\documentclass{article} 
\usepackage{iclr2027_arxiv,times}

\usepackage{amsmath,amsfonts,bm}

\def\eqref#1{equation~\ref{#1}}

\def\1{\bm{1}}

\DeclareMathAlphabet{\mathsfit}{\encodingdefault}{\sfdefault}{m}{sl}
\SetMathAlphabet{\mathsfit}{bold}{\encodingdefault}{\sfdefault}{bx}{n}
\newcommand{\tens}[1]{\bm{\mathsfit{#1}}}

\def\tV{{\tens{V}}}

\usepackage{amsthm}
\newtheorem{theorem}{Theorem}
\newtheorem{definition}{Definition}
\newtheorem{assumption}{Assumption}

\usepackage{hyperref}
\usepackage{url}
\usepackage{graphicx}

\usepackage{xcolor}         
\definecolor{mitred}{RGB}{117, 0, 20}
\usepackage{tcolorbox}

\title{Emergent One-Third Scaling Law as Attention Tries to Concentrate}

\author{Yizhou Liu$^\clubsuit$, Sara Kangaslahti$^{\diamondsuit}$, Jeff Gore$^{\clubsuit}$
\\
Massachusetts Institute of Technology$^\clubsuit$
\\
Harvard University$^{\diamondsuit}$
\\
\texttt{\{liuyz, gore\}@mit.edu}
}

\iclrfinalcopy 
\begin{document}

\maketitle

\begin{abstract}
The neural scaling law relating longer training to better performance through a power law is central to today's large language models (LLMs), yet its origin remains debated. One recent proposal is that power laws can emerge from the strong non-linearity of a single softmax head learning peaked distributions. What happens with multiple softmax functions, as in LLMs, is unclear. Here, we show through toy models that any softmax learning peaked distributions, regardless of its position in the model, can develop logit magnitudes that grow in a power law with exponent $1/3$, becoming a training bottleneck whose loss contribution decays as a power law with the same exponent $1/3$. The overall loss therefore obeys $1/3$ scaling whenever at least one softmax learns peaked distributions. We confirm that many softmax functions in LLMs learn peaked distributions and that LLM loss scaling matches this $1/3$ prediction. Moreover, logit growth dynamics reveal that attention heads, rather than the language modeling head, are the bottleneck likely driving the $1/3$ loss scaling in LLMs. Attention trying to concentrate on specific information, which is the heart of Transformers, may therefore also be the heart of the neural scaling law of training. 
\end{abstract}

\section{Introduction}\label{sec:intro}

Neural scaling laws \citep{hestness2017deep,kaplan2020scaling,hoffmann2022chinchilla} are a key reason large language models (LLMs) are large: loss decreases as a power law with training dataset size and model size, so longer training and more parameters persistently lead to better performance. Despite their empirical power in guiding pre-training, the origins of these power laws remain debated, leaving uncertainty about predicting and improving future scaling.

In this work, we focus on the origin of dataset size scaling. Existing proposals fall into two branches. One attributes the loss to limited information in finite observed samples \citep{sharma2022neural,bahri2024explaining}. The other notes that LLMs use online one-epoch training, where the number of training steps is proportional to dataset size, and explains the scaling law from training dynamics \citep{bordelon2025feature,bordelon2025theory}. Despite the difference, most works from both branches similarly conclude that power laws in loss inherit from power-law data structures.

Recently, a new mechanism was proposed showing that neural scaling can emerge from non-linearity \citep{liu2026universal,kuhn2026boundary} even without power-law structures in the data. When a softmax learns peaked distributions, as happens in LLMs, the non-linearity yields power-law vanishing gradients, ultimately producing power-law training dynamics. Preliminary agreement \citep{liu2026universal,liu2026neural} between the predicted exponent $1/3$ and loss data from open-source models \citep{biderman2023pythia,olmo20242}, as well as the Chinchilla scaling laws \citep{hoffmann2022chinchilla}, suggests the mechanism may be relevant to LLMs. However, more detailed theoretical studies and LLM experiments are needed to verify this connection. In particular, the toy modeling in prior works \citep{liu2026universal,kuhn2026boundary} considered a single layer without feature learning \citep{bordelon2025feature}, while LLMs have multiple layers and softmax functions (e.g., softmax attention). We try to bridge this gap by asking
\begin{tcolorbox}[colframe=mitred, opacityback=0.9, size = title]
\textbf{Question}: How do multiple layers/softmaxes affect the loss landscape and training dynamics?
\end{tcolorbox}
Answering this enables a closer comparison between theory and LLMs and tests the relevance of the proposed mechanism.

By experimenting with toy models, we find that as long as one softmax layer produces peaked distributions, the overall loss decays as a power law with exponent $1/3$. Softmax functions driving this power-law loss also have their logit magnitudes growing as a power law with exponent $1/3$. The theoretical contribution is explaining why softmax in middle layers can lead to the same exponent $1/3$ for the final loss, which follows from a new theorem that this scaling is not specific to loss definitions. In studying LLMs, we find that (i) attention heads and the language modeling (LM) head learn peaked distributions, (ii) some attention heads, not the LM head, have logit magnitudes growing as power laws with exponents close to $1/3$, and (iii) losses decay as power laws with exponents close to $1/3$. These findings suggest that power laws emerging from non-linearity are likely relevant to LLMs, with attention exhibiting the key non-linear effects.
\begin{tcolorbox}[colframe=mitred, opacityback=0.9, size=small, title={To summarize, our contributions are}]
$\bullet$ Any softmax function learning peaked distributions becomes a bottleneck regardless of its location and leads to power-law loss with exponent $1/3$.

$\bullet$ Loss universality: The $1/3$ power-law scaling is not specific to loss definitions, which explains why softmax in middle layers can lead to $1/3$ loss scaling.

$\bullet$ In LLMs, we identified that attention layers trying to concentrate (learning peaked distributions over the context) may be the non-linear effect leading to $1/3$ loss scaling.
\end{tcolorbox}

Overall, we first show that multiple softmax functions still lead to $1/3$ loss scaling and find that attention trying to concentrate may be an origin of neural scaling laws. Toy model setups are in Section~\ref{sec:toy}, theoretical analysis in Section~\ref{sec:theory}, and LLM experiments in Section~\ref{sec:llm}. We compare with related works in Section~\ref{sec:related} and conclude with discussions in Section~\ref{sec:discuss}.

\section{Toy model}\label{sec:toy}
\begin{figure}
\begin{center}
\includegraphics{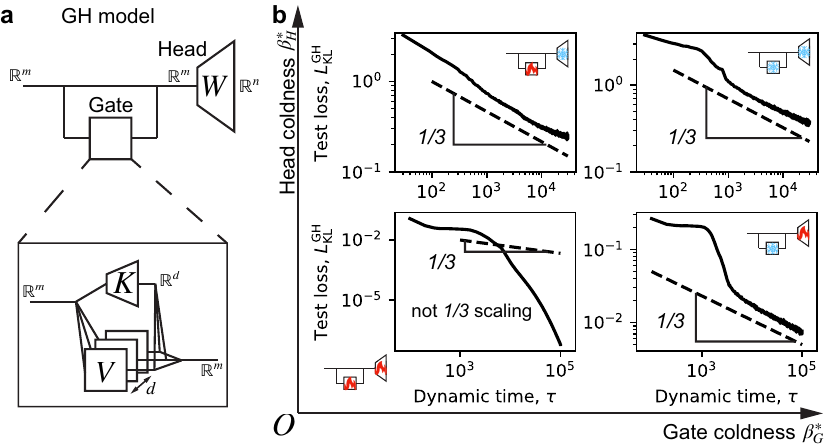}
\end{center}
\caption{In the gate-head (GH) model, any cold softmax, whether the head or the gate, leads to $1/3$ time scaling of loss. (a) The GH model has two parts: a gate layer and a head layer. The gate layer uses the key matrix $K$ and softmax to weight the value matrices. The head layer uses a matrix $W$ and softmax to project hidden states to distributions. (b) The GH model's loss exhibits $1/3$ time scaling whenever the target coldness $\beta_G^*$, $\beta_H^*$, or both are large. Experiment details in Appendix~\ref{sec:GHI}.}
\label{fig:toy}
\end{figure}

To study more than one softmax function as in LLMs while remaining as simple as possible for theoretical and experimental analysis, we propose gate-head (GH) models (Figure~\ref{fig:toy}a). Given input $x\in \mathbb{R}^m$ sampled as i.i.d. standard normal, the GH model outputs a probability distribution $q \in \mathbb{R}^n$,
\begin{equation}
    q(x) = \mathrm{Head}(\mathrm{Norm}(h(x))),~h(x) = \mathrm{Norm}(x) + \mathrm{Gate}(\mathrm{Norm}(x)),
    \label{eq:qhdef}
\end{equation}
where $\mathrm{Norm}(\cdot)$ is root mean square layer normalization. The gate layer contains a key matrix $K\in \mathbb{R}^{d\times m}$ and a value tensor $\tV \in \mathbb{R}^{m \times m \times d}$. Each slice of $\tV$, $\tV_{:,:,i} \in \mathbb{R}^{m \times m}$, is called a value matrix, and we have $d$ value matrices. For any input $x^G\in \mathbb{R}^m$ to the gate, the gate distribution is
\begin{equation}
    q^G = \mathrm{Softmax}(Kx^G) \in \mathbb{R}^d,
\end{equation}
which ``selects" the value matrices to use for the input $x^G$:
\begin{equation}
    \mathrm{Gate}(x^G)_i = \sum_{k} q_k^G \sum_{j} \tV_{ijk} x^G_j,~i = 1,2,...,m.
\end{equation}
Fixing $k$, $(\sum_{j} \tV_{ijk} x^G_j)_i \in \mathbb{R}^m$ is called a value vector. The head contains a matrix $W \in \mathbb{R}^{n \times m}$:
\begin{equation}
    \mathrm{Head}(x^H) = \mathrm{Softmax}(Wx^H) \in \mathbb{R}^n,~\forall x^H \in \mathbb{R}^m.
\end{equation}
To train the GH model, we use a teacher-student setup, where the student is as defined above and is trainable, and the teacher has the same architecture but with fixed parameters. We will use the subscript $*$ to denote teacher parameters (e.g., $W^*$ is the teacher head matrix) and replace $q$ with $p$ for distributions from the teacher ($p^G$ is the teacher distribution in its gate layer). Loss is defined as the Kullback–Leibler (KL) divergence between teacher and student outputs averaged over inputs,
\begin{equation}
    L^{\rm GH}_{\rm KL} = \left \langle \sum_{i=1}^n p_i(x) \ln \frac{p_i(x)}{q_i(x)} \right \rangle_x.
    \label{eq:GHKL}
\end{equation}  

Intuitively, an attention layer selects information from the context by a weighted sum of value vectors of previous tokens. Here, value vectors $(\sum_{j} \tV_{ijk} x_j^G)_i$ from different $k$ are weighted and summed similarly. The head layer copies the architecture of the LM head, which selects the next token to output.

The above GH model is not the only toy model that shares a heuristic similarity with LLMs. We tried different gate architectures, including the original attention, which led to the same phenomena (Appendix~\ref{app:toy}). To gain a concrete understanding, we will focus on a single GH model, after which generalizations will be natural.

The most important quantity to determine the degree of non-linearity is the magnitude of logits (i.e., inputs to softmax). For small logits, softmax can be Taylor expanded with respect to logits and is therefore approximately linear \citep{liu2026universal}. To control the scale of logits, we initialize the teacher key matrix and head matrix as
\begin{equation}
    K^* = \frac{1}{\sqrt{m}}\hat{K}\beta_G^*,~W^* = \frac{1}{\sqrt{m}}\hat{W}\beta_H^*,
    \label{eq:KW}
\end{equation}
respectively, where entries in $\hat{K}$ and $\hat{W}$ are i.i.d. standard normal. Since we have normalization layers, the gate logits and head logits have standard deviations $\beta_G^*$ and $\beta_H^*$, respectively. The teacher value matrices follow LeCun initialization \citep{lecun2002efficient}, not affecting the logit magnitude. Logit standard deviations of the student gate and head will be written as $\beta_G$ and $\beta_H$, respectively. For later convenience, we emphasize
\begin{tcolorbox}[colframe=mitred, opacityback=0.9, size=small, title=Important concepts]
$\bullet$ Logit standard deviation will also be called \textbf{inverse temperature} and \textbf{coldness}.

$\bullet$ A larger logit standard deviation means lower temperature and sharper distributions.
\end{tcolorbox}

Having introduced the GH model, we next study its training dynamics experimentally. We sweep both $\beta_G^*$ and $\beta_H^*$ over a wide range, from $1$ to $1000$, while fixing $m=32$, $d=8$, and $n=128$. We use stochastic gradient descent (SGD) for optimization (details in Appendix~\ref{sec:GHI}). Following \cite{liu2026universal}, we plot the loss against the effective gradient-flow time, the \textbf{dynamic time},
$\tau(t) = \sum_{t'=1}^{t} \eta_{t'},$
where $t$ is the training step and $\eta_{t'}$ is the learning rate at step $t'$. For a single softmax head, sufficiently large coldness leads to the asymptotic loss scaling $\sim\tau^{-1/3}$, i.e., $1/3$ time scaling \citep{liu2026universal,kuhn2026boundary}. We find that the GH model exhibits $\tau^{-1/3}$ scaling whenever $\beta_G^*$, $\beta_H^*$, or both are large. In contrast, when both operate at high temperature, the loss can converge substantially faster than the $1/3$ time scaling (Figure~\ref{fig:toy}b).
The observation suggests the following picture:
\begin{tcolorbox}[colframe=mitred, opacityback=0.9, size=small, title={Result 1: Any softmax can be a training bottleneck for $1/3$ scaling}]
Any softmax in the model, whether in the gate or the head, becomes a training bottleneck when it must produce a sufficiently peaked distribution. If at least one such bottleneck is present, the overall loss cannot decay faster than $1/3$ time scaling.
\end{tcolorbox}

\section{Toy model analysis}\label{sec:theory}

The key new insight should come from the regime where the gate operates at low temperature while the head remains at high temperature. When the gate coldness $\beta_G^*$ is small (left of Figure~\ref{fig:toy}b), the gate softmax can be linearized in its logits, reducing the problem to the previously solved single softmax case \citep{liu2026universal,kuhn2026boundary}. If $\beta_H^*$ is also small, the entire model is approximately linear, and the loss eventually converges exponentially. If instead $\beta_H^*$ is large, the student head must grow in norm to increase its coldness, leading to power-law vanishing gradients and the $1/3$ time scaling \citep{liu2026universal}. The remaining puzzle is therefore why a large gate coldness $\beta_G^*$ can also induce the $1/3$ time scaling (right of Figure~\ref{fig:toy}b). We first study the simplest such regime, with large $\beta_G^*$ and small $\beta_H^*$, before extending the analysis to the case where both are large.

\begin{figure}
\begin{center}
\includegraphics{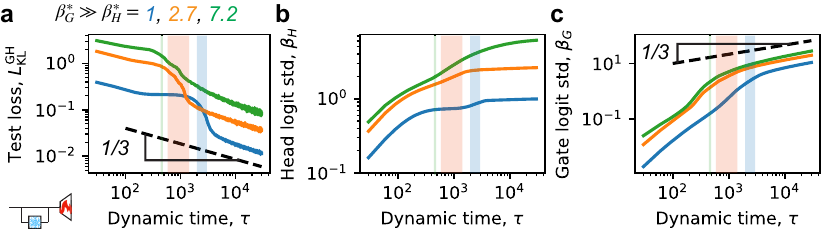}
\end{center}
\caption{When gate coldness is large but head coldness is low, the gate logit magnitude grows as a $1/3$ power law after the head saturates, sustaining $1/3$ loss scaling. (a) The late-time loss follows $1/3$ time scaling across different $\beta_H^*$ (represented by colors). Different loss curves enter their $1/3$ time scaling regions at different times, with the transitions marked by the vertical bands. (b) When loss enters the $1/3$ time scaling region, student head coldness saturates. (c) The $1/3$ scaling in loss also corresponds to student gate coldness growing as a $1/3$ power law. Details in Appendix~\ref{sec:GHI}.}
\label{fig:GH}
\end{figure}

We examine the regime of large $\beta_G^*$ and small $\beta_H^*$ in more detail. Using the same setup as in Figure~\ref{fig:toy}, we additionally track the standard deviations of the student logits as proxies for the weight dynamics (Appendix~\ref{sec:GHI}). Fixing $\beta_G^*=1000$, different small values of $\beta_H^*$ produce loss curves with the $1/3$ time scaling tails but different crossover times (Figure~\ref{fig:GH}a). Interestingly, the student head coldness $\beta_H$ saturates over the same time window in which the loss crosses over to the $1/3$ scaling regime (Figure~\ref{fig:GH}, a and b; colored bands). For $\beta_H^*=7.2$, the crossover is smoother, so we mark only its onset. Because the normalization layer makes $\beta_H$ depend only on the student head matrix $W$, saturation of $\beta_H$ indicates that $W$ has effectively converged. In contrast, the student gate coldness $\beta_G$ continues to grow after head saturation, approximately following a positive $1/3$ time scaling (Figure~\ref{fig:GH}c). These observations suggest that the $1/3$ power-law loss scaling is governed by a still-evolving gate (targeting large $\beta_G^*$) and an already converged head (arriving at small $\beta_H^*$).

We next turn to the gate dynamics. Once the student head $W$ has converged to $W^*$ and the student hidden state $h(x)$ (Eq.~(\ref{eq:qhdef})) is close to the teacher hidden state $h^*(x)$, the leading contribution to the loss $L^{\rm GH}_{\rm KL}$ should be quadratic in $h-h^*$. Its asymptotic scaling should therefore match that of the hidden-state mean squared error (MSE),
\begin{equation}
L^{\rm G}_{\rm Sh}
=
\langle
\|h(x)-h^*(x)\|^2 / m
\rangle_x.
\end{equation}
Thus, to understand the late-time scaling of $L^{\rm GH}_{\rm KL}$ after the head has converged, it is sufficient to study $L^{\rm G}_{\rm Sh}$ in a simplified model where only the gate is trained, i.e., the gate model.

\begin{figure}
\begin{center}
\includegraphics{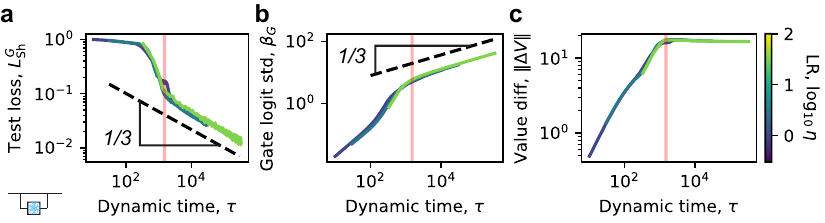}
\end{center}
\caption{Training only the gate with MSE between hidden states yields $1/3$ scaling in both loss and logit magnitude. (a) The loss enters the asymptotic $1/3$ time scaling regime, with the vertical line marking its onset. (b) The student coldness exhibits positive $1/3$ time scaling, beginning at the same time. (c) The value matrix displacement from initialization saturates as the loss and coldness enter their $1/3$ scaling regimes. Curves from different learning rates collapse when plotted against $\tau$, supporting $\tau$ as the fundamental variable governing dynamics. See Appendix~\ref{sec:GI} for details.}
\label{fig:G}
\end{figure}

To test this gate-only reduction, we train the gate using $L^{\rm G}_{\rm Sh}$ (Appendix~\ref{sec:GI}). The loss has the same asymptotic $1/3$ time scaling (Figure~\ref{fig:G}a), supporting the approximation. Meanwhile, the gate coldness grows with positive $1/3$ time scaling (Figure~\ref{fig:G}b). To determine which gate parameters remain active in this regime, we track the value tensor through its displacement from initialization, $\|\tV(\tau)-\tV(\tau = 0)\|$. This quantity saturates at the same time that the loss and gate coldness enter their $1/3$ scaling regimes (Figure~\ref{fig:G}c). Thus, the late-time dynamics are dominated by the non-convergent key matrix $K$, which controls $\beta_G$, while $\tV$ has effectively converged.

Once $\tV=\tV^*$, the moments of the value vectors are constant, so we expect the MSE between hidden states to have the same scaling as that between gate probabilities,
\begin{equation}
L^{\rm G}_{\rm Sh}
\propto L^{\rm G}_{\rm Sp} = 
\langle
\|q^G(x)-p^G(x)\|^2 / d
\rangle_x.
\end{equation}
We can therefore reduce the problem further: the origin of the $1/3$ time scaling in the full GH model can be understood by studying how training $K$ alone under $L^{\rm G}_{\rm Sp}$ yields the same scaling.

We next develop a theory for training $K$ with a large target $\beta_G^*$ (detailed derivations in Appendix~\ref{app:theory}). As explained in \cite{liu2026universal} and \cite{kuhn2026boundary}, when the norm is small, the student matrix $K$ rapidly aligns with the teacher, while the late-time dynamics are dominated by norm growth. We therefore consider the simplified aligned student $K(\tau) = \frac{1}{\sqrt{m}}\hat{K}\beta_G(\tau)$ under gradient flow dynamics, where $K^* = \frac{1}{\sqrt{m}}\hat{K}\beta_G^*$ was defined in Eq.~(\ref{eq:KW}). In this case, $\beta_G(\tau)$ effectively follows its own negative gradient,
\begin{equation}
\frac{\mathrm{d}\beta_G}{\mathrm{d}\tau} = -\frac{1}{d} \frac{\mathrm{d}L^{\rm G}_{\rm Sp}(\beta_G)}{\mathrm{d}\beta_G}.
\label{eq:GF}
\end{equation}
Since $K\propto K^*$, the loss now depends only on $\beta_G$. In the limit of large $\beta_G$ and $\beta_G^* \gg \beta_G$,
\begin{equation}
L^{\rm G}_{\rm Sp}(\beta_G) \approx \left \langle \frac{2e^{-2\beta_G \Delta \epsilon_G}}{d(1 + e^{-\beta_G \Delta \epsilon_G})^2} \right \rangle =\int_0^{\infty}\rho_{\Delta \epsilon_G}(\Delta \epsilon_G) \mathrm{d} \Delta \epsilon_G \frac{2e^{-2\beta_G \Delta \epsilon_G}}{d(1 + e^{-\beta_G \Delta \epsilon_G})^2}.
\end{equation}
Here, $\Delta \epsilon_G$ is the gap between the minimum and second-minimum energies, with energy defined as the normalized logit $\epsilon_G = -y_G / \beta_G$ for logits $y_G$. At large $\beta_G$, only $\Delta \epsilon_G$ contributes appreciably, as higher-energy gaps are exponentially suppressed. Each input therefore contributes one gap $\Delta \epsilon_G$, and the gap distribution $\rho_{\Delta \epsilon_G}(\cdot)$ reflects the data distribution. For sufficiently \textbf{complex and diverse} data, such that the inputs to $K$ vary continuously over the space, we expect $\rho_{\Delta \epsilon_G}(0) \neq 0$, which leads to the expansion
\begin{equation}
L^{\rm G}_{\rm Sp}(\beta_G) =\frac{\rho_{\Delta \epsilon_G}(0)}{\beta_G d}\int_0^{\infty} \mathrm{d} z \frac{2e^{-2z}}{(1 + e^{-z})^2} + o\left(\frac{1}{\beta_G}\right).
\label{eq:expand}
\end{equation}
In our setting, where the logits are i.i.d. Gaussian, $\rho_{\Delta \epsilon_G}(\cdot)$ is approximately an exponential distribution with rate $\sqrt{2\ln d}$. The leading term above therefore dominates when $\beta_G \gg \sqrt{2\ln d}$, making $\sqrt{2\ln d}$ a rough boundary between the high- and low-temperature regimes. In the low-temperature regime $\beta_G^* \gg \beta_G \gg \sqrt{2\ln d}$, $L^{\rm G}_{\rm Sp}(\beta_G)\sim 1/\beta_G$, which together with Eq.~(\ref{eq:GF}) leads to
\begin{equation}
\beta_G \sim \tau^{1/3},~L^{\rm G}_{\rm Sp}\sim \tau^{-1/3}.
\end{equation}
In this idealized theory, the dynamics depend only on $\tau$, consistent with the empirical collapse of curves obtained with different learning rates when plotted against $\tau$ (Figure~\ref{fig:G}).
We therefore identify the continued growth of the key matrix $K$ as the origin of the $1/3$ time scaling. Specifically, the power-law training dynamics arise because the non-linear gate, together with a complex data distribution, produces a loss that intrinsically vanishes as a power law.

The key to the emergent power law is the expansion, Eq.~(\ref{eq:expand}). Previous work \citep{liu2026universal} indicated that this inverse-coldness scaling arises specifically from the KL divergence between probabilities. Here, we show that MSE between probabilities exhibits the same scaling. More generally, replacing the MSE $\langle\|q^G-p^G\|^2/d\rangle$ with $\langle\|q^G-p^G\|^k/d\rangle$ for $k=4,6,8,\ldots$ changes only the integral coefficient (i.e., $\int_0^{\infty} \mathrm{d} z \cdots$ in Eq.~(\ref{eq:expand})), but not the exponent. In fact, the same inverse-coldness scaling holds for all ``reasonable" loss functions:
\begin{tcolorbox}[colframe=mitred, opacityback=0.9, size=small, title={Result 2: Loss universality (informal Theorem~\ref{thm:loss})}]
Consider aligned student and teacher with coldness $\beta$ and $\beta^*$, corresponding probabilities $q$ and $p$, and a non-zero density of zero energy gap. In the regime of large $\beta$ and $\beta^* \gg \beta$, \textbf{any reasonable loss}
$L=\langle D(q,p)\rangle$
scales as
$L\sim 1/\beta$,
\textbf{provided that} $D(q,p)$ is non-negative, vanishes iff $q=p$, and remains differentiable with bounded gradient even when $p$ is one-hot.
\end{tcolorbox}
The conditions for a ``reasonable" loss are mild and probably necessary for trainability. The emergent power law is therefore not a special property of KL divergence or MSE. Instead, it comes from generic ingredients: saturation of probabilities at large coldness and the presence of ambiguous cases in complex data (i.e., $\rho_{\Delta \epsilon_G}(0)\neq 0$).

We can now explain the $1/3$ time scaling of the final $L^{\rm GH}_{\rm KL}$ when the gate is at low temperature and the head is at high temperature more directly. The high-temperature part of the GH model converges rapidly, leaving $L^{\rm GH}_{\rm KL}$ effectively as a function of the non-convergent $q^G$. Although this dependence is complicated, the mapping from $q^G$ to the final probabilities is well-behaved, so $L^{\rm GH}_{\rm KL}$ remains a reasonable loss between $q^G$ and $p^G$. Loss universality gives $L^{\rm GH}_{\rm KL}\sim 1/\beta_G$, and gradient-based training then leads to the $1/3$ time scaling.
Our earlier reduction from $L^{\rm GH}_{\rm KL}$ to $L^{\rm G}_{\rm Sp}$ provided an intuitive explanation but relied on several approximations. Loss universality shows that the precise form of these approximations is irrelevant to the scaling exponent. A low-temperature softmax in a middle layer can therefore yield the same $1/3$ time scaling as a low-temperature softmax at the end.

The case where both $\beta_G^*$ and $\beta_H^*$ are large (Figure~\ref{fig:toy}b, upper right) follows naturally. After alignment, the weights other than $K$ and $W$ converge, while $K$ and $W$ continue to grow in norm and dominate the late-time loss. Analogous to a multi-variable Taylor expansion, $L^{\rm GH}_{\rm KL}$ can be decomposed at leading order into a contribution from imperfect $K$ with $W$ converged and a contribution from imperfect $W$ with $K$ converged (cross terms are higher order at late times). By loss universality, both leading contributions scale as $\tau^{-1/3}$, and therefore so does $L^{\rm GH}_{\rm KL}$.

This argument extends easily beyond the specific GH model. In a model with more softmax functions, each softmax that must learn a peaked distribution can contribute a $\tau^{-1/3}$ term to the late-time loss. Thus, the presence of one or more non-convergent softmax modules, regardless of their positions, can lead to the $1/3$ time scaling of the overall loss. Such bottlenecks in late-time training can be identified by coldness that continues to grow with positive $1/3$ time scaling. We therefore provide an answer to the Question in Section~\ref{sec:intro} on how multiple softmaxes affect the loss scaling.

\section{LLM experiments}\label{sec:llm}
Guided by the mechanism revealed in our toy models, we next ask whether the same picture can explain neural scaling laws in LLMs. Since we mainly focus on the original scaling laws \citep{kaplan2020scaling,hoffmann2022chinchilla}, which are reported from dense models with softmax attention \citep{vaswani2017attention}, we evaluate such models \citep{biderman2023pythia,olmo20242}.

\begin{figure}
\begin{center}
\includegraphics{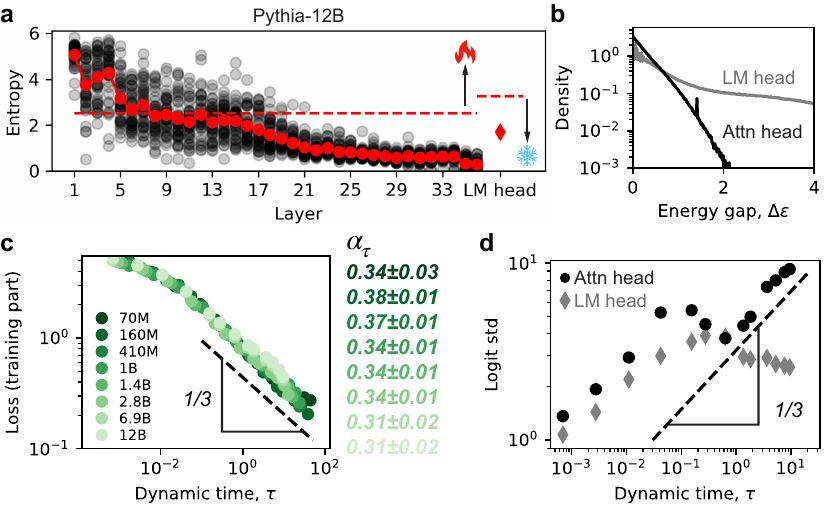}
\end{center}
\caption{In Pythia LLMs, attention heads are more strongly peaked than the LM head, with logit magnitudes growing as $1/3$ power laws, identifying attention as the bottleneck for $1/3$ loss scaling. (a) Entropy suggests that later-layer attention heads and the LM head are in the low-temperature regime (below the horizontal dashed lines). (b) Probability densities at zero energy gap are non-zero. (c) LLMs have losses following the $1/3$ time scaling. (d) Logit coldness of attention heads rather than the LM head grows as the positive $1/3$ time scaling at late training. Details in Appendix~\ref{app:eval}.} 
\label{fig:eval}
\end{figure}

We first ask whether softmax functions in LLMs learn peaked distributions. Our theory predicts a rough boundary between high- and low-temperature regimes for i.i.d. Gaussian logits (Appendix~\ref{sec:expansion}). Because LLM logits need not follow this assumption, we convert this boundary into an entropy threshold and compare it directly with LLM distributions (details in Appendix~\ref{app:eval}). Figure~\ref{fig:eval}a shows results for Pythia-12B \citep{biderman2023pythia}, with other models giving similar results (Appendix~\ref{app:eval}). Dark dots denote the mean attention entropy of individual heads, while red dots show layer averages. The red diamond attached to ``LM head'' marks the entropy of next-token prediction (Appendix~\ref{app:eval}). The dashed line indicates the corresponding high- or low-temperature boundary, with separate thresholds for attention and the LM head. We find that later-layer attention heads and the LM head lie in the low-temperature regime, with some attention heads reaching extremely low entropies.

We next test the second condition required by our theory for the $1/3$ time scaling: the probability density of the energy gap is non-zero at zero. Each input produces a set of logits and hence one energy gap, giving the gap distribution after evaluating the dataset (Appendix~\ref{app:eval}). For both attention heads and the LM head, the observed densities approach non-zero values as the gap tends to zero (Figure~\ref{fig:eval}b). For attention, we show one representative later-layer head from Pythia-12B, with similar behavior across other low-temperature heads (Appendix~\ref{app:eval}). We conclude that the ingredients in our toy models that lead to the $1/3$ time scaling are present in LLMs.

We next test whether LLMs exhibit the predicted power-law training dynamics. Following \cite{liu2026universal}, we evaluate checkpoints of open-source models and fit the raw loss as
\begin{equation}
   L = c_\tau\tau^{-\alpha_\tau} + L_{\backslash \tau},
   \label{eq:lossfit}
\end{equation}
where $L_{\backslash \tau}$ is the part of the loss irreducible by training (Appendix~\ref{app:eval}). For Pythia, we plot the training-dependent component $L-L_{\backslash \tau}$ in Figure~\ref{fig:eval}c, together with the fitted $\alpha_\tau$. Curves from different model sizes collapse when plotted against $\tau$, indicating that $\tau$ is the relevant training variable and that $c_\tau$ and $\alpha_\tau$ are approximately size-independent. The power-law fits are good, with exponents all close to $1/3$.
Our theory also predicts positive $1/3$ time scaling in some logit magnitudes, which is observed (Figure~\ref{fig:eval}d from Pythia-12B; more examples in Appendix~\ref{app:eval}): the growing logits occur in attention heads rather than the LM head, suggesting that attention is the late-training bottleneck. OLMo models \citep{olmo20242} show the same behaviors (Appendix~\ref{app:eval}). Together, all these agreements suggest that our theory may capture the mechanism underlying the neural scaling law of training, with attention as the likely source of the emergent power laws.

\cite{liu2026universal} indicated that the LM head is the bottleneck responsible for the $1/3$ loss scaling. We test this independently by training LLMs with MSE between final logits, which bypasses the output softmax. If that softmax were the only bottleneck, both logits and loss should converge much faster. We train a Pythia-160M student using a trained Pythia-160M model as the teacher providing target logits (Appendix~\ref{app:train}). As expected, the student LM head coldness saturates rather than exhibiting power-law growth (Figure~\ref{fig:distill}a). However, some attention-head coldness continues to grow approximately with positive $1/3$ time scaling, ruling out the LM head as the sole bottleneck. Fitting the MSE loss after LM head saturation (vertical line in Figure~\ref{fig:distill}a) with a power law plus a constant reveals a $1/3$ scaling tail in the training part of the MSE (Figure~\ref{fig:distill}b). These results identify attention as a training bottleneck and directly support loss universality in LLMs.

To summarize, in LLMs we tested the theoretical assumptions and verified predictions for both loss and internal logits. We further ran training experiments with MSE between logits, providing an independent test. The agreement across these analyses supports the following conclusion:
\begin{tcolorbox}[colframe=mitred, opacityback=0.9, size=small, title={Result 3: Relevance to LLMs}]
Power laws emerging from non-linearity are likely relevant to neural scaling laws in LLMs. In particular, attention concentrating toward peaked distributions is likely the late-training bottleneck for the $1/3$ power-law loss scaling.
\end{tcolorbox}

\begin{figure}
\begin{center}
\includegraphics{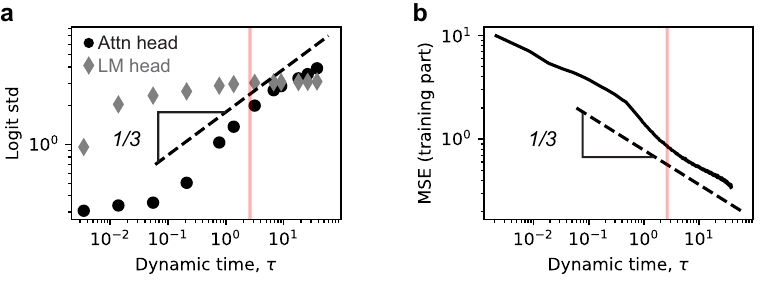}
\end{center}
\caption{Using MSE between final logits for training, attention logit magnitude grows like a positive $1/3$ scaling while the LM head saturates (panel a), ruling out the LM head as the only bottleneck for $1/3$ loss scaling. (b) MSE loss has a tail following the $1/3$ time scaling, as suggested by loss universality. Details in Appendix~\ref{app:train}.}
\label{fig:distill}
\end{figure}

\section{Related works}\label{sec:related}

Empirical studies first established that loss follows approximate power laws \citep{kaplan2020scaling, brown2020gpt3, henighan2020scaling, hoffmann2022chinchilla, openai2023gpt4}, with the Chinchilla scaling laws \citep{hoffmann2022chinchilla} providing more accurate estimates than \cite{kaplan2020scaling} by accounting for irreducible loss. Following \cite{liu2026universal}, in online learning with a fixed learning rate schedule shape, dynamic time is proportional to dataset size. The $1/3$ time scaling can therefore also be viewed as $1/3$ dataset size scaling, consistent with the exponent $0.28\sim0.37$ measured from Chinchilla data \citep{hoffmann2022chinchilla, besiroglu2024chinchilla}.

One branch of theory studies the optimal test loss achievable with a fixed number of training samples. These works attribute power-law scaling to data properties such as manifold dimension \citep{spigler2020asymptotic,hutter2021learning,sharma2022neural} or power-law covariance spectra \citep{bordelon2020spectrum,maloney2022solvable,bahri2024explaining,brill2024unifying}, leading to exponents that depend sensitively on data. Such high-level theories are difficult to test directly in LLMs. A more specific theory based on effective context horizons growing with dataset size \citep{cagnetta2026deriving} is testable, but its data analysis, like that of \cite{kaplan2020scaling}, ignores irreducible loss. The relevance of these theories to LLMs therefore may benefit from further justification.

For online, one-epoch training, theories of training dynamics are conceptually more relevant. Most works analyze linear models \citep{lin2024scaling, bordelon2025theory}, kernels \citep{bordelon2021learning,bordelon2024dynamical,worschech2024analyzing,paquette20244+,bordelon2025feature,defilippis2025scaling}, or high-level learning arguments \citep{michaud2023quantization,arora2023theory,liu2025physicsskilllearning} without treating non-linear effects explicitly. Their common picture is that more important modes are learned first, so power-law structures in the data, such as a power-law covariance spectrum, are required for power-law loss and determine the loss exponent.

Recent work on training dynamics \citep{liu2026universal} instead showed that softmax can yield power-law loss when learning peaked distributions, leading to a robust $1/3$ exponent across diverse data structures even without power laws in data. Our work generalizes this analysis from a single softmax to multiple softmax functions and identifies training attention as the bottleneck for $1/3$ loss scaling in LLMs. \cite{kuhn2026boundary} analyzed the single softmax case more rigorously by incorporating gradient noise and the alignment process, but did not test the theory against LLMs.

Observations of ``attention sinks", ``massive activations", etc. \citep{xiao2024efficient,sun2024massive,zucchet2025emergence,queipo2026attention} suggest that attention tries to concentrate, yet to our knowledge, this work first proposes that these observations may connect to neural scaling laws.

\section{Discussion}\label{sec:discuss}
From toy modeling, we identify learning peaked distributions as a mechanism for emergent power-law loss, regardless of where it occurs in the model. LLM loss and logit dynamics follow the predicted $1/3$ time scaling, suggesting that attention concentrating on specific tokens is the bottleneck underlying this power law. Thus, the tendency of attention to concentrate, a key feature of Transformers, may also underlie the neural scaling law of training.

Our work has several limitations. We focus on scaling with training time, leaving possible effects of non-linearity on model-size scaling unexplored \citep{liu2025superposition,liu2026inverse}. We analyze gradient-flow dynamics, neglecting gradient noise \citep{kuhn2026boundary}, optimizer-specific effects (although Adam shows similar time scaling; Appendix~\ref{app:adam}), and hyperparameter scaling \citep{bergsma2025power}, all of which may affect the mapping between dynamic time and dataset size. We also focus on the original neural scaling laws, while newer architectures such as mixture of experts (MoE) may modify their form \citep{clark2022unified}. For training-time scaling, however, we expect MoE to retain the $1/3$ exponent: expert routing introduces another softmax that may itself become another bottleneck, while increased capacity lowers the achievable loss. These limitations, together with our findings, open directions for future work.

Our results suggest a broader principle for improving LLM scaling. The ``bitter lesson" \citep{sutton2019bitter} favors flexible mechanisms such as attention, which gives LLMs great freedom in modeling context dependence. Yet language may not require such freedom uniformly: only a small fraction of context is often relevant, forcing attention to concentrate and, in our view, producing slow power-law training. Structural priors might therefore accelerate learning by reducing how sharply attention must focus, without sacrificing expressivity. Alternatively, intelligence may fundamentally require selecting a few relevant elements from a vast space of possibilities. Such non-linear concentration, and the associated $1/3$ scaling, may thus be unavoidable. We anticipate that distinguishing between these possibilities could enable more efficient scaling or reveal deeper principles underlying language and intelligence.

\subsection*{AI use statement}

In this work, we used generative AI tools for general brainstorming, formulating mathematical claims, writing code with clear instructions, drafting appendices based on code, and improving the main text readability. We did not use generative AI tools to develop theoretical models or conceptual frameworks, prove mathematical claims, write proofs, or perform data analysis. The rest of the required disclosure tasks are not applicable to this work. We have reviewed all AI-assisted work. We judged AI-generated ideas, verified mathematical claims manually, and wrote proofs manually at the end with our own logic. LLM-generated code was verified and tested for correctness by the authors. Data analysis was done manually by the authors. We take responsibility for the final content of this work, including text, claims, or artifacts produced with the aid of generative AI.

\subsection*{Reproducibility statement}

We explain the logic of experiments and data analysis in appendices, with clear references to the corresponding code files. Code to reproduce all of the results is available at \url{https://github.com/liuyz0/AttnScaling}. For the theoretical results, clear explanations of assumptions and complete proofs of the claims can be found in the appendices.


\bibliography{iclr2027_conference}
\bibliographystyle{iclr2027_conference}

\appendix
\section{Toy model experiments}\label{app:toy}
In the main text, we introduced the GH model and its key results. We hereafter refer to this model as the type-I GH model, or GH-I. In this appendix, we first provide details of the training and analysis of GH-I. We then introduce two additional variants, GH-II and GH-III, which differ from GH-I in their gate architectures. Analysis of GH-I suggests the origin of the emergent $1/3$ scaling as a saturating non-linear effect not specific to its architecture. Experiments on GH-II and GH-III further support the generality of this mechanism, pointing toward more fundamental mathematical principles.

\subsection{GH-I model (the main-text GH model)}\label{sec:GHI}
The basic architecture and setup of the GH-I model have been described in the main text (Section~\ref{sec:toy}).
\subsubsection{Experiment methods}\label{sec:GHI_methods}
The full implementation is in \texttt{exp-0-1.py}.
The architecture is as described in Section~\ref{sec:toy}, with $m=32$, $d=8$, and $n=128$.
The teacher key and head matrices are set according to Eq.~(\ref{eq:KW}), and the teacher value tensor $\tV^*$ follows LeCun initialization.
The student key and head matrices are initialized to zero, and the student value tensor is initialized to a smaller scale than LeCun initialization.

In each training step, inputs $x \in \mathbb{R}^m$ are sampled i.i.d.\ from the standard normal distribution with batch size $2048$. The training loss is the KL divergence $L^{\rm GH}_{\rm KL}$ (Eq.~(\ref{eq:GHKL})). The optimizer is SGD with a constant learning rate $\eta$. We train for $10{,}000$ steps, sweeping $\beta_G^*$ and $\beta_H^*$ each over 8 logarithmically spaced values from $1$ to $1000$, and $\eta$ over $\{3, 10, 30\}$, giving $192$ runs in total. The dynamic time is $\tau = \eta t$ for a constant learning rate, where $t$ is the training step.

Every $10$ steps, the test loss is evaluated on $4$ fresh batches of $2048$ samples and averaged. Three additional quantities are tracked: the student head coldness $\beta_H = \|W\| / \sqrt{n}$, the student gate coldness $\beta_G = \|K\| / \sqrt{d}$, and the value tensor displacement from initialization $\|\tV(t) - \tV(0)\|$. These are consistent with the main text definitions, since the normalization layers make the logit standard deviations proportional to the respective matrix norms. All quantities are saved per run to a result \texttt{.pt} file.

\subsubsection{Data analysis}
All 64 combinations of $(\beta_G^*, \beta_H^*)$ were inspected by plotting $8\times 8$ panels of test loss curves against $\tau$ on log-log axes, with three learning rates overlaid per panel (see \texttt{exp-0-1.ipynb}). The dynamics of $\beta_H$ and $\beta_G$ were similarly scanned across all combinations. These scans confirm that, within the large-$\beta^*$ regime, behaviors across different parameter values are qualitatively consistent. The curves shown in the figures are therefore representative of their regimes, not particular cases.

For Figure~\ref{fig:toy}b, we selected the four corners of the $(\beta_G^*, \beta_H^*)$ sweep, where each axis takes its minimum value $1$ or maximum value $1000$. Each of the four panels plots $L^{\rm GH}_{\rm KL}$ against $\tau$ on log-log axes with one representative learning rate (see \texttt{exp-0-1.ipynb}) and a dashed reference line of slope $-1/3$.

For Figure~\ref{fig:GH}, we fix $\beta_G^* = 1000$ and vary $\beta_H^*$ over its three smallest sweep values, $\beta_H^* \approx 1, 2.7, 7.2$, with learning rate $\eta = 3$. Panels a, b, and c plot $L^{\rm GH}_{\rm KL}$, $\beta_H$, and $\beta_G$ against $\tau$ on log-log axes, respectively. The quantities $\beta_H = \|W\|/\sqrt{n}$ and $\beta_G = \|K\|/\sqrt{d}$ are the saved proxies described in Appendix~\ref{sec:GHI_methods}, with no further processing. Dashed reference lines of slope $-1/3$ and $+1/3$ are drawn for comparison in panels a and c, respectively. The vertical bands marking the transition to the $1/3$ scaling regime are drawn consistently across all three panels. For $\beta_H^* \approx 7.2$, the crossover is smoother and longer, so a single vertical line marks only its onset.

\subsection{G-I model}\label{sec:GI}
The G-I model is the gate layer of GH-I trained in isolation, as described in the main text (Section~\ref{sec:theory}). The full implementation is in \texttt{exp-1-1.py}.

\subsubsection{Experiment methods}\label{sec:GI_methods}
The G-I model uses the gate layer architecture with the same hyperparameters as GH-I ($m=32$, $d=8$), but without the head layer. The teacher key matrix $K^*$ is set according to Eq.~(\ref{eq:KW}), and the teacher value tensor $\tV^*$ follows LeCun initialization. The student $K$ is initialized to zero, and the student $\tV$ is initialized to a smaller scale than LeCun initialization.

In each training step, inputs are sampled i.i.d.\ from the standard normal distribution with batch size $2048$. The training loss is $L^{\rm G}_{\rm Sh}$, the MSE between student and teacher hidden states (main text, Section~\ref{sec:theory}). The optimizer is SGD with constant learning rate $\eta$. We train for $10{,}000$ steps, sweeping $\beta_G^*$ over 8 logarithmically spaced values from $1$ to $1000$ and $\eta$ over $\{0.3, 1, 3, 10, 15, 25, 33, 100\}$, giving $64$ runs in total. The dynamic time is $\tau = \eta t$. Every $10$ steps, the test loss is evaluated on $4$ fresh batches and averaged. The student gate coldness $\beta_G = \|K\|/\sqrt{d}$ and the value tensor displacement $\|\tV(t) - \tV(0)\|$ are also recorded.

\subsubsection{Data analysis}
All 8 values of $\beta_G^*$ were inspected by plotting test loss, $\beta_G$, and $\|\tV(t)-\tV(0)\|$ against $\tau$ on log-log axes, with all 8 learning rates overlaid per panel (see \texttt{exp-1-1.ipynb}). These scans confirm that behaviors in the large-$\beta_G^*$ regime are qualitatively consistent, so the selected curves are representative.

For Figure~\ref{fig:G}, we fix $\beta_G^* = 1000$ and show three representative learning rates $\eta = 1, 3, 33$, color-coded by $\log_{10}\eta$. Panels a, b, and c plot $L^{\rm G}_{\rm Sh}$, $\beta_G$, and $\|\tV(t)-\tV(0)\|$ against $\tau$ on log-log axes, respectively. Dashed reference lines of slope $-1/3$ and $+1/3$ are drawn in panels a and b; panel c has no reference slope, as it shows saturation rather than a power-law trend. The vertical line at $\tau = 1500$ marking the onset of the $1/3$ scaling regime is consistent across all three panels. Curves from different learning rates collapse when plotted against $\tau$, supporting $\tau$ as the fundamental variable governing the dynamics.

\subsection{Adam optimizer}\label{app:adam}

Adam is considerably harder to analyze theoretically than gradient flow, the continuous-time limit that underlies the analysis of SGD. The core difficulty is that Adam maintains per-parameter adaptive learning rates based on gradient moments, making the exact dynamics of $\beta$ analytically intractable. We therefore take an empirical approach and run the same GH-I and G-I experiments under Adam. The results show qualitatively identical phenomena: loss and $\beta$ follow approximate $1/3$ power laws in both models. The effective dynamics of $\beta$ appear similar under Adam and SGD, and differences at the level of individual parameter updates appear to be higher-order corrections that do not change the dominant scaling exponent.

\subsubsection{GH-I extra experiments}

\begin{figure}
\begin{center}
\includegraphics{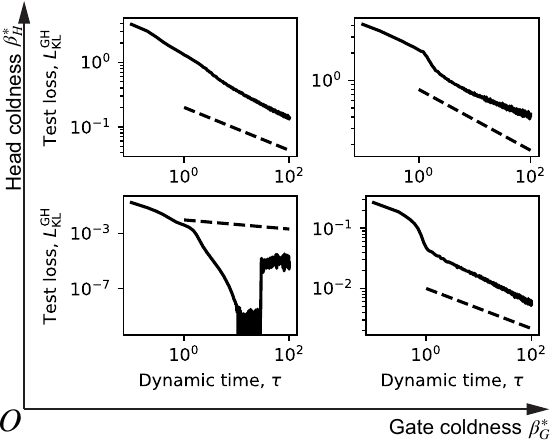}
\end{center}
\caption{Under Adam, the GH model's loss exhibits approximate $1/3$ time scaling whenever the target coldness $\beta_G^*$, $\beta_H^*$, or both are large. Each panel shows the test loss $L_{\rm KL}^{\rm GH}$ against dynamic time $\tau$ on log-log axes, with a dashed reference line of slope $-1/3$. The $2\times 2$ layout arranges the four corners of the $(\beta_G^*, \beta_H^*)$ sweep, with gate coldness increasing rightward and head coldness increasing upward. When both coldnesses are small (bottom left), the loss initially descends rapidly as in the SGD case, but Adam becomes unstable at late times. In the remaining three panels, approximate $1/3$ loss scaling is observed, consistent with the SGD results in Figure~\ref{fig:toy}b.}
\label{fig:toyadam}
\end{figure}

The GH-I Adam experiments use the same architecture, hyperparameters, and sweep design as Appendix~\ref{sec:GHI_methods}, with one modification: Adam replaces SGD. The learning rate is swept over $\{0.01, 0.03, 0.1\}$, giving 192 runs in total. The dynamic time is defined as $\tau = \eta t$, where $\eta$ is the nominal learning rate.

All 64 combinations of $(\beta_G^*, \beta_H^*)$ were inspected by plotting $8\times 8$ panels of test loss curves against $\tau$ on log-log axes, with three learning rates overlaid per panel (see \texttt{exp-0.ipynb}). The dynamics of $\beta_H$ and $\beta_G$ were similarly scanned across all combinations. Curves from different learning rates approximately collapse when plotted against $\tau$, though less tightly than under SGD. Within the large-$\beta^*$ regime, behaviors across different parameter values are qualitatively consistent, so the curves shown in the figures are representative.

Figure~\ref{fig:toyadam} reproduces the four-corner selection of Figure~\ref{fig:toy}b under Adam, using learning rate $\eta = 0.01$ and a dashed reference line of slope $-1/3$. Approximate $1/3$ loss scaling is observed whenever $\beta_G^*$, $\beta_H^*$, or both are large, consistent with the SGD results.

\begin{figure}
\begin{center}
\includegraphics{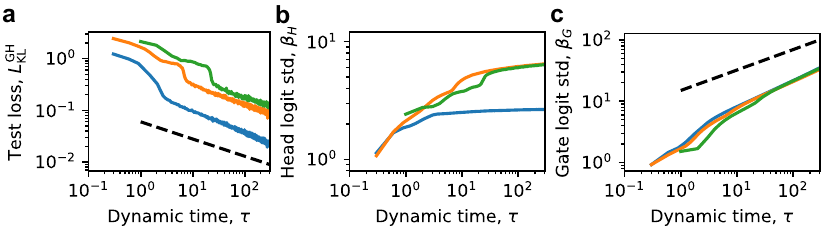}
\end{center}
\caption{Under Adam, when gate coldness is large but head coldness is small, the head saturates and the gate logit magnitude subsequently grows approximately as $\tau^{1/3}$, sustaining approximate $1/3$ loss scaling. Three representative runs are shown, color-coded by run. (a) Test loss $L_{\rm KL}^{\rm GH}$ against $\tau$; the dashed line has slope $-1/3$. (b) Head logit standard deviation $\beta_H$ against $\tau$; $\beta_H$ grows initially then saturates. (c) Gate logit standard deviation $\beta_G$ against $\tau$; the dashed line has slope $+1/3$. The sequence of head saturation followed by gate growth and continued loss decay mirrors the SGD behavior in Figure~\ref{fig:GH}, though the alignment with the reference slopes is less tight.}
\label{fig:GHAdam}
\end{figure}

Figure~\ref{fig:GHAdam} shows three panels plotting $L^{\rm GH}_{\rm KL}$, $\beta_H$, and $\beta_G$ against $\tau$ on log-log axes. Three representative runs are shown: $(\beta_G^*, \beta_H^*, \eta) \in \{(373, 2.7, 0.03),\, (1000, 7.2, 0.03),\, (373, 7.2, 0.1)\}$, all with large $\beta_G^*$ and small $\beta_H^*$, color-coded by blue, orange, and green, respectively. Dashed reference lines of slope $-1/3$ and $+1/3$ are drawn in the loss and $\beta_G$ panels, respectively. As with SGD, $\beta_H$ saturates early and $\beta_G$ subsequently grows approximately as a $\tau^{1/3}$ power law, consistent with the gate becoming the bottleneck.

These results show that the phenomena observed under SGD persist under Adam: approximate $1/3$ loss scaling emerges whenever a softmax learns peaked distributions, and the gate-bottleneck mechanism is visible in the logit dynamics. While the theoretical analysis relies on gradient flow, the empirical agreement under Adam suggests the underlying mechanism may not depend sensitively on the optimizer.

\subsubsection{G-I extra experiments}

The G-I Adam experiments use the same architecture, hyperparameters, and sweep design as the G-I experiment in Appendix~\ref{sec:GI_methods}, with one modification: Adam replaces SGD. The learning rate is swept over $\{0.001, 0.003, 0.01, 0.03, 0.05, 0.08, 0.1, 0.3\}$, giving 64 runs in total. The dynamic time is defined as $\tau = \eta t$.

All 8 values of $\beta_G^*$ were inspected by plotting test loss, $\beta_G$, and $\|\tV(t)-\tV(0)\|$ against $\tau$ on log-log axes, with all 8 learning rates overlaid per panel (see \texttt{exp-1.ipynb}). Curves from different learning rates approximately collapse when plotted against $\tau$. Within the large-$\beta_G^*$ regime, behaviors are qualitatively consistent, so the selected curves are representative.

\begin{figure}
\begin{center}
\includegraphics{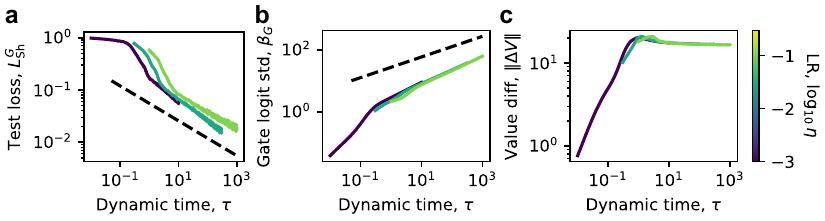}
\end{center}
\caption{Under Adam, training only the gate with MSE loss yields approximate $1/3$ scaling in both loss and gate logit magnitude, mirroring the SGD results in Figure~\ref{fig:G}. We fix $\beta_G^* = 1000$ and show three representative learning rates, color-coded by $\log_{10}\eta$ from dark purple to light green. (a) Test loss $L_{\rm Sh}^{\rm G}$ against $\tau$; the dashed line has slope $-1/3$. (b) Gate logit standard deviation $\beta_G$ against $\tau$; the dashed line has slope $+1/3$. (c) Value tensor displacement $\|\tV(t)-\tV(0)\|$ against $\tau$; $\|\tV(t)-\tV(0)\|$ rises and then saturates. Curves from different learning rates collapse well when plotted against $\tau$.}
\label{fig:GAadm}
\end{figure}

For Figure~\ref{fig:GAadm}, we fix $\beta_G^* = 1000$ and show three representative learning rates $\eta = 0.001$, $0.03$, and $0.1$, color-coded by $\log_{10}\eta$ from dark purple to light green. Panels a, b, and c plot $L^{\rm G}_{\rm Sh}$, $\beta_G$, and $\|\tV(t)-\tV(0)\|$ against $\tau$ on log-log axes. Dashed reference lines of slope $-1/3$ and $+1/3$ are drawn in panels a and b; panel c has no reference slope. As with SGD, loss decays and $\beta_G$ grows approximately as $\tau^{1/3}$ power laws, and $\|\tV(t)-\tV(0)\|$ rises and saturates.

The G-I model under Adam reproduces the same phenomena as under SGD: approximate $1/3$ loss scaling and $\beta_G$ growth driven by the gate. Together with the GH-I Adam results, this confirms that the $1/3$ scaling phenomena are not specific to SGD.

\subsection{GH-II model}\label{sec:GHII}

\begin{figure}
\begin{center}
\includegraphics{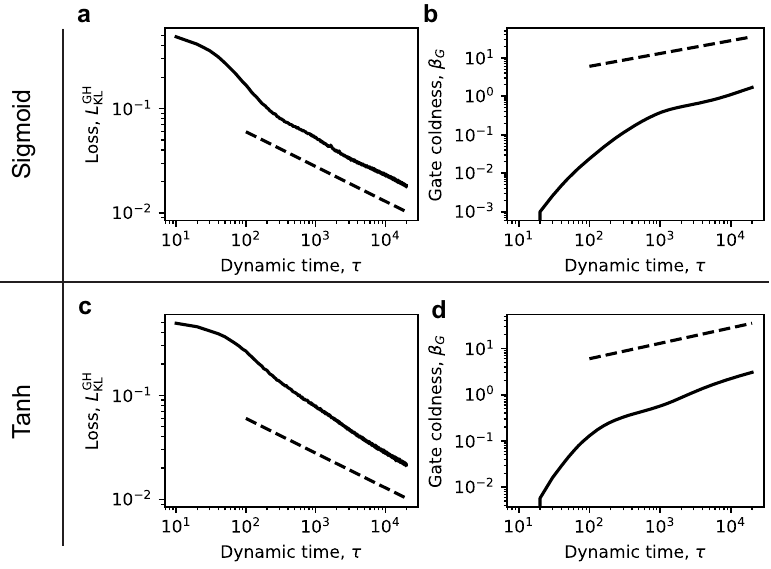}
\end{center}
\caption{The GH-II model reproduces $1/3$ time scaling in loss and gate coldness for both sigmoid and tanh gate non-linearities ($\beta_G^*=1000$, $\beta_H^*=1$). Each pair of panels shows $L^{\rm GH}_{\rm KL}$ (left) and $\beta_G$ (right) against $\tau$ on log-log axes. (a,b) Sigmoid variant. (c,d) Tanh variant. Dashed reference lines have slope $-1/3$ (panels a,c) and $+1/3$ (panels b,d).}
\label{fig:GHII}
\end{figure}

The GH-I model (the main-text GH model) and its gate-only reduction established that any low-temperature softmax in the model can become the training bottleneck, producing $1/3$ time scaling in loss. The theoretical explanation (Section~\ref{sec:theory}) hinges on loss universality and the $1/\beta_G$ loss scaling in the low-temperature regime, both of which depend on the non-linear property of the gate, not on its specific form. This raises the question of whether the same phenomena appear when the gate uses a different architecture. GH-II tests this with a modified gate, using the same head and training protocol as GH-I.

GH-II uses the same overall structure as GH-I (Eq.~(\ref{eq:qhdef})), including the head layer and residual connection. The gate function $\mathrm{Gate}(\cdot)$ is modified: it now contains a key matrix $K \in \mathbb{R}^{d \times m}$, a gate matrix $G \in \mathbb{R}^{d \times m}$, and a value matrix $V \in \mathbb{R}^{m \times d}$, with output
\begin{equation}
    \mathrm{Gate}(x^G) = V\bigl(Kx^G \odot \sigma(Gx^G)\bigr),
\end{equation}
where $\odot$ denotes elementwise multiplication and $\sigma$ is applied elementwise. Two variants are studied: $\sigma = \mathrm{Sigmoid}$ (\texttt{test-2-1.ipynb}) and $\sigma = \mathrm{Tanh}$ (\texttt{test-2-2.ipynb}), with initialization and learning rate details in the notebooks. The gate coldness is $\beta_G = \|G\|/\sqrt{d}$, and the teacher gate coldness $\beta_G^*$ controls the scale of the teacher gate matrix $G^*$. The coldness $\beta_G$ therefore controls the magnitude of $Gx^G$ before the non-linearity, determining how saturated the outputs of $\sigma$ will be. We use $m=32$, $d=96$, $n=128$, and fix $\beta_G^*=1000$. The training loss is $L^{\rm GH}_{\rm KL}$, the optimizer is SGD with a constant learning rate $\eta$, the dynamic time is $\tau = \eta t$, and we log the loss and $\beta_G$.

We run both GH-II variants in the regime of large gate coldness ($\beta_G^* = 1000$) and small head coldness ($\beta_H^* = 1$). Both the $\mathrm{Sigmoid}$ and $\mathrm{Tanh}$ variants exhibit $\tau^{-1/3}$ loss scaling (Figure~\ref{fig:GHII}, panels a and c) and $\tau^{1/3}$ gate coldness growth (panels b and d), reproducing the same phenomena as GH-I. The $1/3$ scaling is therefore more general, extending beyond the specific softmax non-linearity.

The Sigmoid results admit a direct theoretical interpretation: $\mathrm{Sigmoid}(y) = e^y/(e^y+1)$ is exactly the softmax probability of the more likely class in a binary problem, so each of the $d$ gate components is a 2-class softmax. The theory of Section~\ref{sec:theory} therefore applies. With $\beta_G^*=1000$ controlling the scale of all rows of $G^*$, all $d$ sigmoid gates are in the low-temperature regime. Once $K$ and $V$ converge, the overall loss is a reasonable function of the $d$ gate activations $\mathrm{Sigmoid}(Gx^G)$. Loss universality then gives $L^{\rm GH}_{\rm KL}\sim 1/\beta_G$, which together with gradient flow dynamics yields $\tau^{-1/3}$ loss and $\tau^{1/3}$ gate coldness growth, as observed. The Tanh variant follows by the same argument: $\mathrm{Tanh}(y) = 2\mathrm{Sigmoid}(2y)-1$ is a linear rescaling of sigmoid, so it shares the same saturation behavior as a function of $\beta_G$, and the low-temperature expansion and loss universality apply identically.

In summary, GH-II replaces the softmax gate of GH-I with sigmoid and tanh non-linearities and finds that both reproduce $1/3$ time scaling in loss and gate coldness. Theoretical analysis connects these results directly to the softmax case via the binary-softmax equivalence and loss universality. Taken together with GH-I, these three concrete examples (softmax, sigmoid, and tanh) consistently exhibit the same $1/3$ scaling across different gate architectures and non-linearities. This suggests a \textbf{non-linearity universality} analogous to loss universality: just as the $1/3$ scaling does not depend on the loss form, it also does not depend on the specific non-linearity or gate architecture, pointing toward a more general mathematical principle whose complete characterization is a direction for future work. 

Notably, the GH-II gate is structurally a gated linear unit (GLU), which is widely adopted in modern LLMs as a feed-forward layer variant. If such a GLU gate operates in the low-temperature regime, it would become yet another training bottleneck, making the $1/3$ loss scaling harder to escape. The benefit of GLU over standard feed-forward layers may lie in reducing the irreducible loss through more expressive gating. However, convergence to that lower loss floor may still follow the same $1/3$ power law.

\subsection{GH-III model}

GH-I and GH-II demonstrated the $1/3$ time scaling with gate architectures that, while realistic, are simpler than the full attention mechanism used in LLMs. GH-III asks whether the same phenomena persist when the gate is replaced by real attention, providing a more direct connection between the toy model and the original Transformers.

GH-III uses the same overall structure as GH-I (the main text GH model), including the head layer and residual connection, but takes $x \in \mathbb{R}^{m \times d}$ as input, where $d$ is the context length ($d=8$ in our experiments). The gate function $\mathrm{Gate}(\cdot)$ is replaced by single-head attention computing the output for the last token. The input to the gate layer is $x^G = \mathrm{Norm}(x) \in \mathbb{R}^{m \times d}$, and the residual connection and the head act on the last token only, so $h = \mathrm{Norm}(x_d) + \mathrm{Gate}(x^G) \in \mathbb{R}^m$, where $x_d$ is the last token. A combined weight matrix $W^{QKV} \in \mathbb{R}^{3m \times m}$ projects each token of $x^G$ into query, key, and value vectors, giving $Q = [\boldsymbol{q}_1,\ldots,\boldsymbol{q}_d]$, $K = [\boldsymbol{k}_1,\ldots,\boldsymbol{k}_d]$, $V = [\boldsymbol{v}_1,\ldots,\boldsymbol{v}_d] \in \mathbb{R}^{m \times d}$ (here $K$ and $V$ denote the stacked key and value vectors of the tokens, not the weight matrices of GH-I). The gate output is
\begin{equation}
    \mathrm{Gate}(x^G) = V\,\mathrm{Softmax}\!\left(\frac{K^T \boldsymbol{q}_d}{\sqrt{m}}\right) \in \mathbb{R}^m,
\end{equation}
where $\boldsymbol{q}_d$ is the last token's query. The gate coldness $\beta_G$ is the standard deviation of the attention logits, averaged over the batch, and the teacher coldness is $\beta_G^*=1000$. We use $m=32$, $d=8$, $n=128$, $\beta_H^*=1$, and train with AdamW with dynamic time $\tau=\eta t$. The loss and $\beta_G$ are logged, with initialization and learning rate details in \texttt{test-3-3.ipynb}.

We run GH-III in the regime of large gate coldness ($\beta_G^* = 1000$) and small head coldness ($\beta_H^* = 1$). The loss exhibits $\tau^{-1/3}$ scaling (Figure~\ref{fig:GHIII}a) and the gate coldness grows as $\tau^{1/3}$ (Figure~\ref{fig:GHIII}b), reproducing the same phenomena as GH-I. The $1/3$ time scaling therefore persists with real attention as the gate.

The GH-III gate is conceptually more complex than GH-I: the attention logits involve both $K$ and $Q$, and the $d$-dimensional context introduces additional structure absent in GH-I's single-vector input. These complexities are one reason GH-I, with its simpler gate, was chosen as the main-text model for theoretical analysis. However, once the student $Q$, $K$, $V$ matrices are aligned with the teacher (ensured here by initialization), the extra structure does not affect the leading-order behavior. The loss reduces to a function of $\beta_G$ alone, and the low-temperature expansion gives $L^{\rm GH}_{\rm KL} \sim 1/\beta_G$ by the same argument as in GH-I. If $\beta_G$ then follows effective gradient flow dynamics, the same $\tau^{1/3}$ growth follows naturally, explaining the $1/3$ scaling in both $\beta_G$ and loss.

In summary, GH-III confirms that the $1/3$ time scaling persists when the gate is replaced by real attention, with the theoretical argument carrying over through the aligned-student assumption. GH-I was presented in the main text because its simpler gate makes the essential mechanism (a low-temperature softmax as a training bottleneck) transparent and analytically tractable. Once this mechanism is understood from GH-I, the extension to GH-III follows naturally, and GH-III then provides the direct experimental link between the toy model and real Transformer attention. The resulting logic chain, from GH-I to GH-III to the LLM experiments in the main text, forms a coherent and increasingly realistic case that attention trying to concentrate is at the origin of the emergent $1/3$ scaling law.

\begin{figure}
\begin{center}
\includegraphics{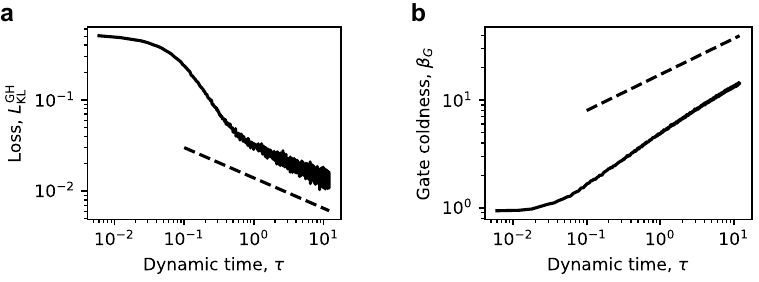}
\end{center}
\caption{The GH-III model with real attention as the gate reproduces $1/3$ time scaling in loss and gate coldness ($\beta_G^*=1000$, $\beta_H^*=1$). (a) The loss $L^{\rm GH}_{\rm KL}$ follows $\tau^{-1/3}$ scaling. (b) The gate coldness $\beta_G$ grows as $\tau^{1/3}$. Dashed reference lines have slope $-1/3$ (panel a) and $+1/3$ (panel b).}
\label{fig:GHIII}
\end{figure}

We also note a structural connection between GH-I (main text GH model) and MoE architectures: the GH-I gate uses $q^G = \mathrm{Softmax}(Kx^G)$ to weight $d$ value matrices, in the same way MoE routing selects among experts. If the routing distribution becomes peaked as the model learns to specialize experts, our analysis predicts that the routing gate will become a training bottleneck with a $1/3$ time-scaling contribution to the loss. As with GLU (Appendix~\ref{sec:GHII}), the benefit of MoE is increased model capacity and a lower achievable loss. However, convergence to that lower loss floor may still follow the same $1/3$ power law. This supports the prediction in Section~\ref{sec:discuss} that MoE architectures are likely to retain the $1/3$ scaling exponent.

\section{Theory}\label{app:theory}
In this section, we provide full details of our theoretical derivations.
Appendices \ref{sec:gflow}, \ref{sec:aligned}, and \ref{sec:expansion} elaborate on the gate dynamics theory in Section~\ref{sec:theory}.
Appendix~\ref{sec:theorem} generalizes the setup and proves loss universality.

\subsection{Gradient flow}\label{sec:gflow}
In Appendices \ref{sec:gflow}, \ref{sec:aligned}, and \ref{sec:expansion}, we analyze the G-I model (gate layer of GH-I or the main text GH model) with all notations used identical to the main text. As described in Section~\ref{sec:theory}, we focus on the training dynamics of the key matrix $K$ using MSE between gate probabilities, $L^{\rm G}_{\rm Sp}$, as loss. Essentially, our analysis is on a linear layer followed by a softmax non-linearity, similar to \cite{liu2026universal} and \cite{kuhn2026boundary}.

The gradient descent update of the student $K$ is given by
\begin{equation}
    K_{t+1} = K_t - \eta_t \nabla_{K} L^{\rm G}_{\rm Sp},
\end{equation}
where $\nabla_K$ takes the gradient with respect to $K$.
In the continuous limit of small $\eta_t$, with $\tau(t) = \sum_{t'=1}^{t} \eta_{t'} \approx \int_0^t \eta_{t'}\mathrm{d}t'$, the step $\eta_t$ plays the role of $\mathrm{d}\tau$, and we obtain the gradient flow dynamics
\begin{equation}
    \frac{\mathrm{d}K}{\mathrm{d}\tau}= - \nabla_{K} L^{\rm G}_{\rm Sp}.
\end{equation}

\subsection{Aligned student}\label{sec:aligned}
As explained in \cite{liu2026universal}, when the norm is small, the student matrix $K$ rapidly aligns with the teacher, while the late-time dynamics are dominated by norm growth. We therefore consider the simplified aligned student $K(\tau) = \frac{1}{\sqrt{m}}\hat{K}\beta_G(\tau)$ under gradient flow dynamics, where $K^* = \frac{1}{\sqrt{m}}\hat{K}\beta_G^*$ was defined in Eq.~(\ref{eq:KW}). To derive the effective dynamics of $\beta_G$, we substitute this aligned $K$ into the gradient flow of $K$. We then have
\begin{equation}
    \frac{\hat{K}}{\sqrt{m}}\frac{\mathrm{d}\beta_G}{\mathrm{d}\tau}= - \nabla_{K} L^{\rm G}_{\rm Sp},
\end{equation}
and then
\begin{equation}
    \frac{\mathrm{Tr}(\hat{K}\hat{K}^T)}{m}\frac{\mathrm{d}\beta_G}{\mathrm{d}\tau}= - \mathrm{Tr}\left(\nabla_{K} L^{\rm G}_{\rm Sp}\frac{\hat{K}^T}{\sqrt{m}}\right).
\end{equation}
Noting that $\frac{\mathrm{d}K}{\mathrm{d}\beta_G} = \frac{\hat{K}}{\sqrt{m}}$, the right side of the above equation is $\nabla_{\beta_G} L^{\rm G}_{\rm Sp}$,
\begin{equation}
    \frac{\mathrm{Tr}(\hat{K}\hat{K}^T)}{m}\frac{\mathrm{d}\beta_G}{\mathrm{d}\tau}= - \nabla_{\beta_G} L^{\rm G}_{\rm Sp}.
\end{equation}
We consider large $d$ and $m$, such that $\mathrm{Tr}(\hat{K}\hat{K}^T) \approx md$. We therefore reach
\begin{equation}
    \frac{\mathrm{d}\beta_G}{\mathrm{d}\tau}= - \frac{1}{d}\frac{\mathrm{d} L^{\rm G}_{\rm Sp}(\beta_G)}{\mathrm{d} \beta_G}.
    \label{eq:appGF}
\end{equation}

\subsection{Low temperature expansion}\label{sec:expansion}
The key is to understand $L^{\rm G}_{\rm Sp}(\beta_G)$ now. We study the limit of large $\beta_G$ and $\beta_G^* \gg \beta_G$ (low temperatures). Recall that
\begin{equation}
L^{\rm G}_{\rm Sp} = 
\langle
\|q^G-p^G\|^2 / d
\rangle = \Big \langle
\frac{1}{d}\sum_{i=1}^d (q^G_i-p^G_i)^2
\Big \rangle,
\end{equation}
where $\langle \cdot \rangle$ means averaging over the data distribution. The student's gate probability distribution is
\begin{equation}
q^G_i = \frac{e^{y_{G,i}}}{\sum_j e^{y_{G,j}}} = \frac{e^{-\beta_G \epsilon_{G,i}}}{\sum_j e^{-\beta_G \epsilon_{G,j}}},
\end{equation}
and the teacher's probability distribution is
\begin{equation}
p^G_i = \frac{e^{y^*_{G,i}}}{\sum_j e^{y^*_{G,j}}} = \frac{e^{-\beta_G^* \epsilon_{G,i}}}{\sum_j e^{-\beta_G^* \epsilon_{G,j}}}.    
\end{equation}
Here, $y_G = Kx^G \in \mathbb{R}^d$ are the gate logits, with $x^G \in \mathbb{R}^m$ as the gate input. Energies $\epsilon_G = -y_G/\beta_G$ are i.i.d. standard normal. Since the teacher and student are aligned, their energies are the same, and only the standard deviations $\beta_G$ and $\beta_G^*$ differ.

For each $\{\epsilon_{G,i}\}_{i=1}^d$ or $\epsilon_G = [\epsilon_{G,1},...,\epsilon_{G,d}]^T$, we use $\epsilon_G^{\uparrow}$ for the corresponding sorted energies, such that $\min_i \epsilon_{G,i} = \epsilon^{\uparrow}_{G,1} \leq \epsilon^{\uparrow}_{G,2} \leq \cdots \leq \epsilon^{\uparrow}_{G,d}$. We can define energy gaps as $\Delta \epsilon_{G,i}= \epsilon^{\uparrow}_{G,i} - \epsilon^{\uparrow}_{G,1}$. In particular, we define
\begin{equation}
    \Delta \epsilon_{G} \equiv \Delta \epsilon_{G,2}.
\end{equation}

With these preparations, we can write
\begin{equation}
    \sum_i (q^G_i)^2 = \frac{1 + e^{-2\beta_G \Delta\epsilon_{G}} +\sum_{i=3}^d e^{-2\beta_G \Delta\epsilon_{G,i}}}{(1+e^{-\beta_G \Delta\epsilon_{G}}+\sum_{i=3}^d e^{-\beta_G \Delta\epsilon_{G,i}})^2}.
\end{equation}
The condition $\beta^*_G\gg \beta_G$ effectively means any term involving $\beta^*_G$ is higher-order, and we can use $\beta^*_G = \infty$ to only keep the leading terms. We then have $\sum_i (p^G_i)^2 = 1$ and
\begin{equation}
    \sum_i p^G_i q^G_i = \frac{1}{1+e^{-\beta_G \Delta\epsilon_{G}}+\sum_{i=3}^d e^{-\beta_G \Delta\epsilon_{G,i}}}.
\end{equation}
The loss can then be written as
\begin{align}
L^{\rm G}_{\rm Sp} &= 
\frac{1}{d}\Big \langle
\frac{2e^{-2\beta_G \Delta\epsilon_{G}}(1 + O(e^{-\beta_G(\epsilon^{\uparrow}_{G,3}- \epsilon^{\uparrow}_{G,2})}))}{(1 + e^{-\beta_G \Delta\epsilon_{G}} + \sum_{i=3}^d e^{-\beta_G \Delta\epsilon_{G,i}})^2}
\Big \rangle\\
& = \frac{1}{d}\Big \langle
\frac{2e^{-2\beta_G \Delta\epsilon_{G}}}{(1 + e^{-\beta_G \Delta\epsilon_{G}})^2}\Big \rangle + \frac{1}{d}\Big \langle
\frac{2e^{-2\beta_G \Delta\epsilon_{G}}}{(1 + e^{-\beta_G \Delta\epsilon_{G}})^2}O(e^{-\beta_G(\epsilon^{\uparrow}_{G,3}- \epsilon^{\uparrow}_{G,2})})\Big \rangle.\label{eq:Lexpand}
\end{align}

The first term in Eq.~(\ref{eq:Lexpand}) is the leading term, such that in terms of the scaling with $\beta_G$, the second term can only be the same or faster. To find the leading scaling with $\beta_G$, it suffices to study the scaling of the first term.
As in the main text,
\begin{equation}
\left \langle \frac{2e^{-2\beta_G \Delta \epsilon_G}}{(1 + e^{-\beta_G \Delta \epsilon_G})^2} \right \rangle =\int_0^{\infty}\rho_{\Delta \epsilon_G}(\Delta \epsilon_G) \mathrm{d} \Delta \epsilon_G \frac{2e^{-2\beta_G \Delta \epsilon_G}}{(1 + e^{-\beta_G \Delta \epsilon_G})^2}.
\end{equation}
For sufficiently \textbf{complex and diverse} data, such that the inputs to $K$ vary continuously over the space, we expect $\rho_{\Delta \epsilon_G}(0) \neq 0$, which leads to the low-temperature expansion
\begin{align}
\int_0^{\infty}\rho_{\Delta \epsilon_G}(\Delta \epsilon_G) \mathrm{d} \Delta \epsilon_G \frac{2e^{-2\beta_G \Delta \epsilon_G}}{(1 + e^{-\beta_G \Delta \epsilon_G})^2}&= \int_0^{\infty}\rho_{\Delta \epsilon_G}\Big(\frac{z}{\beta_G}\Big) \frac{\mathrm{d} z}{\beta_G} \frac{2e^{-2z}}{(1 + e^{-z})^2}\\
&=\frac{\rho_{\Delta \epsilon_G}(0)}{\beta_G}\int_0^{\infty} \mathrm{d} z \frac{2e^{-2z}}{(1 + e^{-z})^2} + o\left(\frac{1}{\beta_G}\right).\label{eq:lowexpand}
\end{align}
In the specific setup of our toy model, where energies are i.i.d. standard normal, the gap distributions can be obtained via the limiting Poisson point process for extreme order statistics in the Gumbel class \citep{resnick1987extreme}. For $\Delta \epsilon_G$, $\rho_{\Delta \epsilon_G}(\cdot)$ is approximately an exponential distribution with rate $\sqrt{2\ln d}$. The first term in Eq.~(\ref{eq:lowexpand}) therefore dominates when $\beta_G \gg \sqrt{2\ln d}$, making $\sqrt{2\ln d}$ a rough boundary between the high- and low-temperature regimes.\footnote{With the same argument, for the head layer, where the output dimension is $n$, the rough boundary between the high- and low-temperature regimes is given by $\beta_H = \sqrt{2\ln n}$.}

Finally, we can have the loss as
\begin{equation}
    L^{\rm G}_{\rm Sp} = \frac{1}{d}\left(\frac{\rho_{\Delta \epsilon_G}(0)}{\beta_G}\int_0^{\infty} \mathrm{d} z \frac{2e^{-2z}}{(1 + e^{-z})^2} + o\left(\frac{1}{\beta_G}\right) \right) (1 + O(1)),
\end{equation}
where the $O(1)$ comes from the second term of Eq.~(\ref{eq:Lexpand}). The second term in Eq.~(\ref{eq:Lexpand}) is actually $O(\langle e^{-\beta_G\Delta\epsilon_{G,3}}\rangle)$. With extra information from extreme order statistics of i.i.d. standard normal random variables, $\rho_{\Delta \epsilon_{G,k}}(\Delta \epsilon_{G,k}) \sim (\Delta \epsilon_{G,k})^{k-2}$, we can obtain $O(\langle e^{-\beta_G\Delta\epsilon_{G,3}}\rangle) = O(1/\beta_G^2)$, similar to what we did for the first term of Eq.~(\ref{eq:Lexpand}). In our specific setup, we can further tighten the loss expression by knowing $O(1)$ above as $O(1/\beta_G)$, yielding
\begin{equation}
    L^{\rm G}_{\rm Sp} = \frac{\rho_{\Delta \epsilon_G}(0)}{\beta_G d}\int_0^{\infty} \mathrm{d} z \frac{2e^{-2z}}{(1 + e^{-z})^2} + o\left(\frac{1}{\beta_G}\right).
\end{equation}

With or without the extra information about the extreme value statistics, once $\rho_{\Delta \epsilon_G}(0) \neq 0$, we have the scaling for large $\beta_G$:
\begin{equation}
    L^{\rm G}_{\rm Sp} = \Theta \left( \frac{1}{\beta_G} \right).
\end{equation}
Combined with the effective gradient flow dynamics of $\beta_G$, i.e., Eq.~(\ref{eq:appGF}), we have
\begin{equation}
    \beta_G = \Theta(\tau^{1/3}),~\text{thus}~L^{\rm G}_{\rm Sp}=\Theta(\tau^{-1/3}).
\end{equation}

The power-law result above is correct when $\beta_G^* \gg \beta_G$. When $\beta_G$ approaches $\beta_G^*$, one can always expand the loss around $\beta_G^*$. The leading term in the loss should then be $\Theta((\beta_G-\beta_G^*)^2)$, where the coefficient is related to the Hessian. We therefore have exponential convergence. When $\beta_G^*$ is large, the full picture is then: (i) when $\beta_G$ is small, there is no power-law loss; (ii) when $\beta_G$ is large while still much smaller than $\beta_G^*$, loss enters the $1/3$ power-law regime; (iii) when $\beta_G$ approaches $\beta_G^*$, loss will deviate from the power law and converge exponentially. If $\beta_G^* = \infty$, the $1/3$ scaling will last forever. In \cite{liu2026universal}, the $1/3$ power-law regime is called the ``intermediate regime".

We focus on dynamics at low temperatures in this paper. On the other hand, when both $\beta_G^*$ and $\beta_G$ are small, the softmax function can be expanded with respect to the logits directly, i.e., the high-temperature expansion \citep{liu2026universal}. The leading terms in the high-temperature expansion are linear in logits, and the model is therefore effectively linear. In our specific setup, where the inputs are isotropic, the loss converges exponentially in the high-temperature regime. 

\subsection{Loss universality}\label{sec:theorem}

In the above appendices, we analyzed a single softmax inside the model. We find that training such a softmax to learn peaked distributions will lead to a part of the loss scaling as $\tau^{-1/3}$. If there are more softmax functions operating at low temperatures, each of them can contribute a $\tau^{-1/3}$ term to the late-time loss. The interaction between softmax functions is higher-order if each is already close to being optimal. The late-time loss therefore scales as $\tau^{-1/3}$ as long as there exist softmax functions approaching peaked distributions.

How the final loss specifically depends on the softmax outputs is affected by the position of the softmax. For the final softmax, the final loss is the KL divergence of output probabilities, which shows $1/3$ scaling \citep{liu2026universal,kuhn2026boundary}. For intermediate softmax functions, the final loss is more similar to MSE between their outputs (after later layers converge), which also scales as $\tau^{-1/3}$. These examples therefore suggest that $1/3$ time scaling is not specific to how the loss depends on softmax probabilities. 

Here, we solidify intuitions into a theorem. We will first introduce the setup and notations valid and self-consistent only in this Appendix~\ref{sec:theorem}.
\begin{definition}[Problem setup]\label{def:setup}
    Consider a trainable student module $q = \mathrm{Softmax}(y(x)) \in \mathbb{R}^d$ with logits $y$ depending on model input $x$. It tries to approach a teacher $p = \mathrm{Softmax}(y^*(x)) \in \mathbb{R}^d$ where $y^*(\cdot)$ is a fixed mapping by minimizing loss $L = \langle D(q,p) \rangle$. Here, $D(\cdot,\cdot)$ is a measure of difference and $\langle \cdot \rangle$ denotes averaging over the data distribution. Without loss of generality, we set $\sum_i y_i = 0$ and $\sum_i y^*_i = 0$.
\end{definition}
This setup is general. If we imagine the module is the final softmax and $D(\cdot,\cdot)$ is KL divergence, it can describe toy models in \cite{liu2026universal}. And if we imagine the module as an intermediate softmax, once all later layers converge, the loss only depends on this softmax ($D(\cdot,\cdot)$ will depend on parameters of the later layers and cannot be written as a simple function), which is the gate-layer training that the current paper focuses on.

We next introduce the assumptions facilitating further analysis. As in the main text, we define coldness $\beta$ as the standard deviation of $y$ and $\beta^*$ as the standard deviation of $y^*$. The student is often aligned in early training:
\begin{assumption}[Aligned student]\label{asu:align}
    Under the setup (Definition~\ref{def:setup}), the student has $y\propto y^*$ for all inputs, such that the loss only depends on $\beta$ and $\beta^*$.
\end{assumption}
We also define energies $\epsilon = - y / \beta = - y^* / \beta^*$ under aligned student. 
If the logits are obtained by multiplying a hidden state by a matrix, and the data is complex enough such that hidden states are distributed continuously across the hidden space, the probability density of hidden states on decision boundaries of output classes will be non-zero, which inspires the following generic assumption.
\begin{assumption}[Complex and diverse data]\label{asu:data}
    Under the setup (Definition~\ref{def:setup}) and with the aligned student (Assumption~\ref{asu:align}), for each $\{\epsilon_i\}_{i=1}^d$ or $\epsilon = [\epsilon_1,...,\epsilon_d]^T$, we use $\epsilon^{\uparrow}$ for the corresponding sorted energies, such that $\min_i \epsilon_i = \epsilon^{\uparrow}_1 \leq \epsilon^{\uparrow}_2 \leq \cdots \leq \epsilon^{\uparrow}_d$. We write energy gaps as $\Delta \epsilon_i = \epsilon^{\uparrow}_i - \epsilon^{\uparrow}_1 \geq 0$ and denote $\Delta \epsilon \equiv \Delta \epsilon_2$. The data distribution affects the probability density $\rho_{\Delta \epsilon_i}(\cdot)$ of gap $\Delta \epsilon_i$. Data are sufficiently complex and diverse such that $\rho_{\Delta \epsilon}$ is continuous near zero and $\rho_{\Delta \epsilon}(0) \neq 0$.
\end{assumption}

We next explain the constraints on the loss functions.
\begin{definition}[Reasonable loss]\label{def:loss}
    Under the setup (Definition~\ref{def:setup}), the loss function $L = \langle D(q,p) \rangle$ is called reasonable if the following conditions are satisfied:
    
    $\bullet$ (Regular) $D(q,p)$ is permutation invariant (i.e., exchanging any two coordinates of $q$ and $p$ does not change $D(q,p)$), non-negative, and vanishes if and only if the two distributions agree. $D(q,p)$ remains finite and continuous when the teacher $p$ is a one-hot distribution for an aligned student.

    $\bullet$ (Differentiable) $D(q,p)$ is Lipschitz continuous in $q$ on the probability simplex for any $p$ when the student is aligned. $D(q,p)$ is differentiable in $q$ on the probability simplex excluding the one-hot vertices for any $p$ when the student is aligned. The loss $L$ is differentiable in model parameters.
\end{definition}
There are several remarks. The requirement that $D(q,p)=0$ iff $q=p$ is not necessary but convenient. For example, cross-entropy converges to the entropy of $p$ when $q=p$, which may not be zero. But it is easy to define cross-entropy minus the entropy of $p$, i.e., the KL divergence, to satisfy the condition. The behaviors of how cross-entropy converges to its final value (the entropy) and how KL divergence converges to zero are the same. And we only care about how the reducible part of the loss vanishes. So, it is convenient to set $D(q,p)=0$ iff $q=p$. But any $D(q,p)$ that can be modified to satisfy this condition, like cross-entropy, is also ``reasonable".
We think the constraints define reasonable losses as they basically require the loss to be finite and to have finite gradients, which are desired by healthy gradient-based optimization.
It is straightforward to check that commonly used MSE, KL divergence, and polynomials of $\|p-q\|^2$ are reasonable.

With all the preparations, we can now state the claim.
\begin{theorem}[Loss universality]\label{thm:loss}
    Under the setup (Definition~\ref{def:setup}), assuming aligned student (Assumption~\ref{asu:align}) and complex data (Assumption~\ref{asu:data}), any reasonable loss $L$ (Definition~\ref{def:loss}) satisfies $L = \Theta(\beta^{-1})$ and $\nabla_{\beta} L = \Theta(\beta^{-2})$ for sufficiently large $\beta$ and $\beta^* \gg \beta$.
\end{theorem}
The result $L = \Theta(\beta^{-1})$ is the key to $1/3$ time scaling. If $\beta$ follows gradient flow effectively, the gradient is $\Theta(\beta^{-2})$ and time integral yields $\beta = \Theta(\tau^{1/3})$, which leads to $L = \Theta(\tau^{-1/3})$.

\begin{proof}
We next prove Theorem~\ref{thm:loss}. Introducing
\begin{equation}
    q^{\uparrow} = \mathrm{Softmax}(-\beta \epsilon^{\uparrow})~\text{and}~p^{\uparrow} = \mathrm{Softmax}(-\beta^* \epsilon^{\uparrow})
\end{equation}
as the sorted student and teacher probabilities, respectively, we have
\begin{equation}
    D(q,p) = D(q^{\uparrow},p^{\uparrow}).
\end{equation}
We denote
\begin{equation}
    q^{\circ}_z = \left(
\frac{1}{1+e^{-z}},
\frac{e^{-z}}{1+e^{-z}},
0,\ldots,0
\right)^T~\text{and}~\boldsymbol{e}_1 = (1, 0, 0, \ldots,0)^T.
\end{equation}
The condition $\beta^* \gg \beta$ is equivalent to substituting $\beta^*=\infty$, ignoring high-order terms about $\beta^*$.
\begin{equation}
    D(q,p) = D(q^{\uparrow},\boldsymbol{e}_1) = D(q^{\circ}_{\beta\Delta \epsilon},\boldsymbol{e}_1) + (D(q^{\uparrow},\boldsymbol{e}_1)-D(q^{\circ}_{\beta\Delta \epsilon},\boldsymbol{e}_1)).
\end{equation}
Since $D(\cdot,\cdot)$ is differentiable,
\begin{equation}
    \frac{D(q^{\uparrow},\boldsymbol{e}_1)-D(q^{\circ}_{\beta\Delta \epsilon},\boldsymbol{e}_1)}{D(\boldsymbol{e}_1,\boldsymbol{e}_1)-D(q^{\circ}_{\beta\Delta \epsilon},\boldsymbol{e}_1)} = O(e^{-\beta(\epsilon^{\uparrow}_3-\epsilon^{\uparrow}_2)}) = O(1).
\end{equation}
The $O(1)$ bound is very loose, including the worst case that $\epsilon^{\uparrow}_3-\epsilon^{\uparrow}_2=0$. 
It is then safe to write
\begin{equation}
    L = \langle D(q^{\uparrow},\boldsymbol{e}_1) \rangle = \langle D(q^{\circ}_{\beta\Delta \epsilon},\boldsymbol{e}_1)\rangle(1 + O(1)).
\end{equation}
It suffices to study the scaling of $\langle D(q^{\circ}_{\beta\Delta \epsilon},\boldsymbol{e}_1)\rangle$.

Averaging data is equivalent to averaging over $\Delta \epsilon$ for $ D(q^{\circ}_{\beta\Delta \epsilon},\boldsymbol{e}_1)$,
\begin{align}
    \langle D(q^{\circ}_{\beta\Delta \epsilon},\boldsymbol{e}_1)\rangle &= \int_0^\infty D(q^{\circ}_{\beta\Delta \epsilon},\boldsymbol{e}_1) \rho_{\Delta \epsilon}(\Delta \epsilon) \mathrm{d}\Delta \epsilon\\
    &= \int_0^\infty D(q^{\circ}_{z},\boldsymbol{e}_1) \rho_{\Delta \epsilon}(z/\beta) \frac{\mathrm{d}z}{\beta}.
\end{align}
We next need to study the property of $\lambda(z) \equiv D(q^{\circ}_{z},\boldsymbol{e}_1)$.\footnote{For KL divergence, $\lambda(z) = \ln(1+e^{-z})$; for MSE, $\lambda(z) = \frac{2e^{-2z}}{(1 + e^{-z})^2d}$.} Since $D(\cdot,\cdot)$ is Lipschitz, we know that there exists some constant $C_\lambda$ such that
\begin{equation}
    \lambda(z) = D(q^{\circ}_{z},\boldsymbol{e}_1) \leq C_\lambda e^{-z}.
\end{equation}
This guarantees that 
\begin{equation}
    \int_0^\infty \lambda(z) \mathrm{d}z < \infty.
\end{equation}
It is therefore valid to write
\begin{align}
    \int_0^\infty D(q^{\circ}_{z},\boldsymbol{e}_1) \rho_{\Delta \epsilon}(z/\beta) \frac{\mathrm{d}z}{\beta} &= \int_0^\infty \lambda(z) (\rho_{\Delta \epsilon}(0)+o(1)) \frac{\mathrm{d}z}{\beta}\\
    &=\frac{\rho_{\Delta \epsilon}(0)}{\beta}\int_0^\infty \lambda(z)\mathrm{d}z + o\left(\frac{1}{\beta}\right),
\end{align}
where Assumption~\ref{asu:data} is also used.
We can obtain
\begin{equation}
    L = \left(\frac{\rho_{\Delta \epsilon}(0)}{\beta}\int_0^\infty \lambda(z)\mathrm{d}z + o\left(\frac{1}{\beta}\right)\right)(1 + O(1)),
\end{equation}
which is equivalent to
\begin{equation}
    L = \Theta(\beta^{-1}),~\beta\to \infty~\&~\beta/\beta^*\to 0.
\end{equation}

For an arbitrary function that is $\Theta(\beta^{-1})$, its derivative is not necessarily $\Theta(\beta^{-2})$ as it can be very oscillatory. But $\nabla_{\beta} L = \Theta(\beta^{-2})$ as will be shown. Intuitively, fast decay of $\lambda(z)$ as a result of Lipschitz continuity of $D(\cdot,\cdot)$ and continuity of $\rho_{\Delta \epsilon}$ prevents strong oscillation of $L$.
Formally,
\begin{equation}
    |\lambda'(z)| \leq C_\lambda e^{-z},
\end{equation}
and
\begin{equation}
    \int_0^\infty z \lambda'(z) \mathrm{d}z < \infty.
\end{equation}
Integrating by parts, we have
\begin{equation}
    \int_0^\infty z \lambda'(z) \mathrm{d}z = - \int_0^\infty \lambda(z) \mathrm{d}z.
\end{equation}
We can then have
\begin{align}
    \nabla_{\beta} \langle D(q^{\circ}_{\beta\Delta \epsilon},\boldsymbol{e}_1)\rangle &= \int_0^\infty \nabla_{\beta}\lambda(\beta \Delta \epsilon) \rho_{\Delta \epsilon}(\Delta \epsilon) \mathrm{d}\Delta \epsilon\\
    &= \int_0^\infty \lambda'(z) z \rho_{\Delta \epsilon}(z/\beta) \frac{\mathrm{d}z}{\beta^2}\\
    & = -\frac{\rho_{\Delta \epsilon}(0)}{\beta^2}\int_0^\infty \lambda(z)\mathrm{d}z + o\left(\frac{1}{\beta^2}\right),
\end{align}
which leads to $\nabla_{\beta} L = \Theta(\beta^{-2}),~\beta\to \infty~\&~\beta/\beta^*\to 0$.
\end{proof}

\section{LLM experiments}

\subsection{Evaluation of open-source models}\label{app:eval}

\subsubsection{Evaluation methods}\label{sec:C.1.1}

We evaluate the Pythia \citep{biderman2023pythia} and OLMo-2 \citep{olmo20242} model families, using text from the Pile \citep{gao2020pile} and FineWeb \citep{penedo2024fineweb} for Pythia and from FineWeb for OLMo-2.
For each family, we load checkpoints at training steps that are approximately logarithmically spaced, spanning early training to the final checkpoint.
The same procedures are applied uniformly across all model sizes. The checkpoint steps, listed from early to late for each model, are specified in the evaluation scripts referenced below, and the specific models are reported alongside the results in Appendix~\ref{sec:C.1.3}.
Three quantities are measured at each checkpoint, described in turn below.

\paragraph{Loss and LM head logits.}
We compute the cross-entropy loss and the LM head logit coldness at each checkpoint (\texttt{pythia-logit-1.py} for Pythia and \texttt{olmo-logit-0.py} for OLMo-2).
For each checkpoint, we perform a forward pass on batches of text tokenized to a context length of $1024$, evaluating on approximately $5 \times 10^6$ tokens in total.
Valid token positions are those that are non-padding and have a non-padding successor (since the loss is next-token prediction).
The cross-entropy loss is averaged over valid tokens.
The LM head logit coldness is the standard deviation of the pre-softmax LM head logits over the vocabulary, averaged over valid token positions.

\paragraph{Attention logit coldness and entropy.}
We compute the per-head pre-softmax attention score standard deviation and softmax entropy across training checkpoints (\texttt{eval\_attn\_logit\_std.py} for Pythia and \texttt{olmo/eval\_attn\_logit\_std.py} for OLMo-2).
For each checkpoint, we evaluate on approximately $2 \times 10^5$ tokens with context length $2048$.
For each attention layer and head, we extract the pre-softmax attention scores $s_{ij} = (\boldsymbol{q}_i \cdot \boldsymbol{k}_j) / \sqrt{d_{\rm head}}$ with rotary embeddings applied, where $i$ and $j$ denote query and key positions, respectively.
Causal masking restricts each query position $i$ to the $i+1$ valid key positions $0, 1, \ldots, i$.
For each query position $i \geq 1$, we compute the unbiased standard deviation of the scores over its valid key positions, then average over query positions and evaluation batches to obtain the per-head logit coldness.
We also compute the entropy (in nats) of the resulting softmax attention distribution, averaged in the same way.

\paragraph{Energy gap distribution.}
We compute the energy gap for both attention heads and the LM head at each checkpoint (\texttt{eval\_energy\_gap.py} for Pythia and \texttt{olmo/eval\_energy\_gap.py} for OLMo-2).
For a logit vector $y$, the energy gap is
\begin{equation}
    \Delta\epsilon = \frac{y_{(1)} - y_{(2)}}{\mathrm{std}(y)},
\end{equation}
where $y_{(1)}$ and $y_{(2)}$ are the largest and second-largest logits.
For attention, $y$ consists of the pre-softmax scores over valid causal key positions, using the same masking as above for $i \geq 1$.
For the LM head, $y$ is the logit vector over the full vocabulary.
Gaps are accumulated as a fixed-bin histogram over approximately $2 \times 10^6$ tokens per checkpoint, which yields the distribution needed to check whether the density at $\Delta\epsilon \to 0$ is non-zero, as required by the theory (main text, Section~\ref{sec:theory}).

\subsubsection{Analysis methods}\label{sec:C.1.2}

\paragraph{Dynamic time.}
The dynamic time $\tau(t) = \sum_{t'=1}^{t} \eta_{t'}$ (main text) is computed by numerically integrating the learning rate schedule.
For Pythia, the schedule uses a linear warmup for the first $0.01\,t_{\max}$ steps, followed by a cosine decay:
\begin{equation}
    \eta_{t'} = \eta_{\rm peak}\bigl[0.9\cdot\tfrac{1}{2}(\cos\theta_{t'}+1)+0.1\bigr], \quad
    \theta_{t'} = \frac{t'-0.01\,t_{\max}}{0.99\,t_{\max}}\,\pi,
\end{equation}
where $t_{\max}$ is the total number of training steps and $\eta_{\rm peak}$ varies by model size.
For OLMo-2, the warmup covers the first $t_{\rm warm}=2000$ steps, followed by the same cosine form:
\begin{equation}
    \eta_{t'} = \eta_{\rm peak}\bigl[0.9\cdot\tfrac{1}{2}(\cos\theta_{t'}+1)+0.1\bigr], \quad
    \theta_{t'} = \frac{t'-t_{\rm warm}}{t_{\max}-t_{\rm warm}}\,\pi.
\end{equation}
In both cases, $\tau(t)$ is obtained by cumulative summation of $\eta_{t'}$ over steps $t'=1,\ldots,t$, evaluated at the selected checkpoint steps.
Since $\eta_{\rm peak}$ differs across model sizes, the resulting $\tau$ places all model sizes on a common dynamic-time axis, enabling the loss-curve collapse shown in Figure~\ref{fig:eval}c.

\paragraph{Entropy threshold.}
Our theory predicts a rough boundary between the high- and low-temperature regimes for i.i.d.\ Gaussian logits: a softmax with $n$ logits drawn from $\mathcal{N}(0, \beta^2)$ crosses the boundary at $\beta = \sqrt{2\ln n}$ (Appendix~\ref{sec:expansion}).
Because LLM logits are not i.i.d.\ Gaussian, we convert this criterion into an entropy threshold and compare it directly with entropies of LLM distributions.
For attention, the number of valid key positions grows from $1$ to the context length as the query position advances.
For each number of valid key positions $i$ from $2$ to $2047$, we set $\beta = \sqrt{2\ln i}$, draw $i$ logits from $\mathcal{N}(0, \beta^2)$, and compute the entropy (in nats) of the resulting softmax distribution.
The attention entropy threshold $\bar{H}_{\rm attn}$ is the mean of these values over $i$, representing the average boundary entropy at the transition.
For the LM head, we apply the same procedure with $n$ equal to the vocabulary size, giving the threshold $\bar{H}_{\rm LM}$.
These two thresholds appear as horizontal dashed lines in Figure~\ref{fig:eval}a.
For attention heads, we compare the measured entropy with $\bar{H}_{\rm attn}$, and heads below the threshold are in the low-temperature regime.
For the LM head, we instead compare the target entropy of next-token prediction with $\bar{H}_{\rm LM}$: at convergence, the LM head must reproduce the next-token distribution of the data, whose entropy is bounded above by the fitted irreducible loss of $1.69$ nats from the Chinchilla scaling laws \citep{hoffmann2022chinchilla}.
This upper bound is the red diamond in Figure~\ref{fig:eval}a and already lies below $\bar{H}_{\rm LM}$, so the target entropy is lower still, and the LM head must learn peaked distributions and is in the low-temperature regime.
Results for Pythia-410M and Pythia-160M are shown in Figure~\ref{fig:pythia-entropy}.

\paragraph{Loss fitting.}
Following \cite{liu2026universal}, we fit the raw loss at each model size independently to
\begin{equation}
    L(\tau) = c_\tau\,\tau^{-\alpha_\tau} + L_{\backslash\tau},
\end{equation}
where $c_\tau$, $\alpha_\tau$, and $L_{\backslash\tau}$ are free parameters (\texttt{pythia-logit-0to7-v3.ipynb}).
Fitting is performed by nonlinear least squares on $\log L$ (rather than $L$) to reduce the influence of early-training checkpoints where the loss is large, using \texttt{scipy.optimize.curve\_fit}.
The uncertainties reported with $\alpha_\tau$ in Figures~\ref{fig:eval}c and \ref{fig:olmo2}c are standard errors, the square roots of the diagonal of the covariance matrix returned by the solver.
Early checkpoints (typically the first one or two per model) are excluded from the fit because the model is still in the rapid alignment phase and has not yet entered the power-law regime.
The fitted irreducible loss $L_{\backslash\tau}$ is subtracted from all checkpoints to yield the training-dependent component $L - L_{\backslash\tau}$, which is plotted in Figure~\ref{fig:eval}c together with the fitted exponents $\alpha_\tau$. The raw loss and the full fitting procedure for all eight Pythia model sizes are shown in Figure~\ref{fig:rawloss}.

\paragraph{Attention logit std and energy gap density.}
The per-head attention logit coldness (standard deviation of pre-softmax scores) is plotted against $\tau$ on logarithmic axes (\path{attn_logit_std_*.ipynb}).
A reference line of slope $+1/3$ is overlaid for visual comparison. No fitting of the attention logit std is performed.
Figure~\ref{fig:eval}d shows selected heads from Pythia-12B (\path{attn_logit_std_12b.ipynb}): the logit coldness of certain attention heads grows approximately as $\tau^{1/3}$ while the LM head coldness saturates, identifying attention as the late-training bottleneck. Results for all heads across all layers of Pythia-12B are shown in Figure~\ref{fig:12b-all-coldness}.
For the energy gap, the raw histogram counts are divided by the product of the bin width and the total sample count to obtain a probability density estimate (\path{energy_gap_*.ipynb}).
Figure~\ref{fig:eval}b shows results from Pythia-12B (\path{energy_gap_12b.ipynb}): the density remains non-zero as $\Delta\epsilon \to 0$ for both a representative attention head and the LM head, confirming the theoretical condition (main text, Section~\ref{sec:theory}). Results for all heads across all layers of Pythia-12B are shown in Figure~\ref{fig:12b-all-gaps}.
Energy gap and logit coldness results for Pythia-410M and Pythia-160M are shown in Figure~\ref{fig:pythia-gap-coldness}.
The OLMo-2-13B analysis, covering entropy (panel a), energy gap density (panel b), loss scaling (panel c), and attention logit coldness (panel d), is shown in Figure~\ref{fig:olmo2}.

\subsubsection{Additional results}\label{sec:C.1.3}

Figure~\ref{fig:pythia-entropy} extends the peaked distribution analysis of Figure~\ref{fig:eval}a to Pythia-410M and Pythia-160M.
The three models (160M, 410M, and 12B) were chosen because their depths span a wide range: 12, 24, and 36 layers, respectively.
The per-head entropy at the final checkpoint is computed by the same procedure as in Appendix~\ref{sec:C.1.2} (entropy threshold paragraph), using \path{attn_logit_std_410m.ipynb} for Pythia-410M and \path{attn_logit_std.ipynb} for Pythia-160M.
In both models, low-temperature attention heads (below the entropy threshold) exist across the depth of the network, and the upper bound of the LM head target entropy also falls below its threshold.
On average, later layers tend to have lower entropy, but concentrated heads appear in early layers as well.
Our theory requires only that at least one low-temperature softmax exists, and this condition is satisfied across all three model sizes, indicating that the peaked distribution pattern is not an artifact of model scale.

\begin{figure}
\begin{center}
\includegraphics{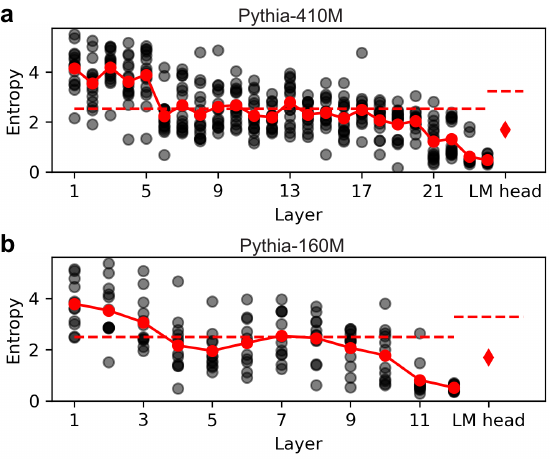}
\end{center}
\caption{Later-layer attention heads and the LM head are in the low-temperature regime for Pythia-410M (a) and Pythia-160M (b). Same format as Figure~\ref{fig:eval}a: black dots show per-head entropy at the final checkpoint, red dots show layer averages, and the red diamond shows the upper bound of the LM head target entropy ($1.69$ nats). Dashed lines mark the entropy thresholds for attention heads (left) and the LM head (right). The pattern is consistent with Pythia-12B in the main text.}
\label{fig:pythia-entropy}
\end{figure}

Figure~\ref{fig:12b-all-gaps} extends the energy gap analysis of Figure~\ref{fig:eval}b from a single representative head to all attention heads across all 36 layers of Pythia-12B (\path{energy_gap_12b.ipynb}; same normalization procedure as Appendix~\ref{sec:C.1.2}).
For most heads and layers, the density is non-zero at $\Delta\epsilon \to 0$, indicating that the theoretical condition is not confined to a few exceptional heads.
Some heads in early layers show density approaching zero at small $\Delta\epsilon$, consistent with their higher entropies in Figure~\ref{fig:eval}a.
This confirms that the energy gap condition required by the theory is broadly satisfied across the depth of Pythia-12B.

\begin{figure}
\begin{center}
\includegraphics[width = 0.9\linewidth]{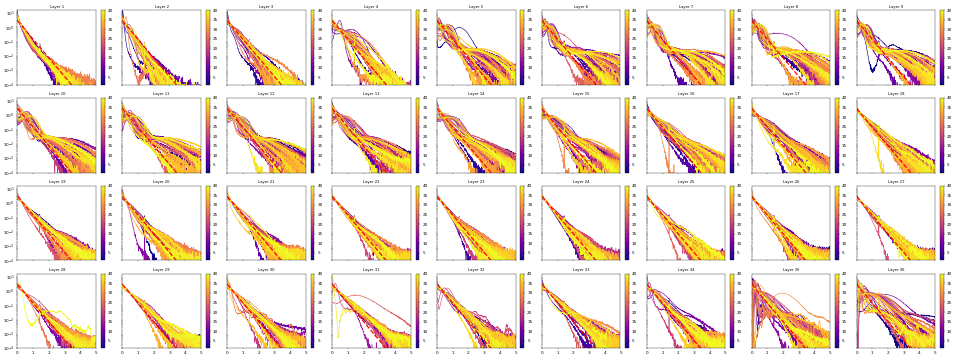}
\end{center}
\caption{Energy gap density distributions for all attention heads across all 36 layers of Pythia-12B (final checkpoint), extending Figure~\ref{fig:eval}b. Each panel corresponds to one layer; curves are colored by head index. For most heads across all layers, the density is non-zero at $\Delta\epsilon \to 0$, confirming the theoretical condition broadly. Heads with density approaching zero at small $\Delta\epsilon$ are in the high-temperature regime. Analysis follows Appendix~\ref{sec:C.1.2}.}
\label{fig:12b-all-gaps}
\end{figure}

Figure~\ref{fig:12b-all-coldness} extends the logit coldness analysis of Figure~\ref{fig:eval}d to all attention heads across all 36 layers of Pythia-12B (\path{attn_logit_std_12b.ipynb}; same procedure as Appendix~\ref{sec:C.1.2}).
Across the network, the low-temperature heads identified in Figure~\ref{fig:eval}a broadly show logit coldness growing at approximately the $+1/3$ rate at late training, indicating that their logit magnitudes have not converged.
This ongoing growth is precisely what the theory predicts for a bottleneck softmax: its logit magnitude must keep growing to track the teacher, producing $1/3$ time scaling in its loss contribution.
The behavior is therefore not peculiar to the representative head shown in Figure~\ref{fig:eval}d but is generic among low-temperature heads, further supporting attention as the late-training bottleneck.

\begin{figure}
\begin{center}
\includegraphics[width = 0.9\linewidth]{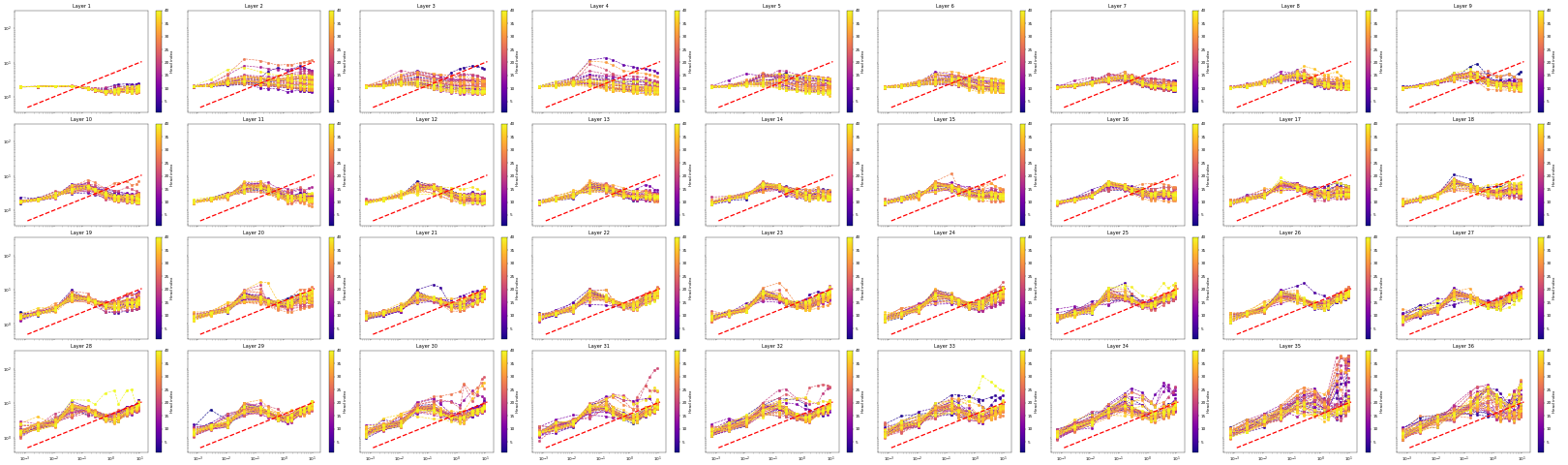}
\end{center}
\caption{Attention logit coldness versus dynamic time $\tau$ for all attention heads across all 36 layers of Pythia-12B, extending Figure~\ref{fig:eval}d. Each panel corresponds to one layer; curves are colored by head index. The dashed red line has slope $+1/3$ for reference. Low-temperature heads broadly track the $+1/3$ slope at late training, while high-temperature heads grow more slowly or saturate. Analysis follows Appendix~\ref{sec:C.1.2}.}
\label{fig:12b-all-coldness}
\end{figure}

Figure~\ref{fig:pythia-gap-coldness} replicates the energy gap and logit coldness analyses of Figure~\ref{fig:eval}b and d for Pythia-410M (panels a--b) and Pythia-160M (panels c--d), using \path{energy_gap_410m.ipynb}, \path{energy_gap_160m.ipynb}, \path{attn_logit_std_410m.ipynb}, and \path{attn_logit_std.ipynb} with the same procedures as Appendix~\ref{sec:C.1.2}.
In both models, the energy gap density of representative attention heads is non-zero at $\Delta\epsilon \to 0$, and the logit coldness of attention heads grows approximately as $\tau^{1/3}$ while the LM head coldness saturates.
These results reproduce the findings of Figure~\ref{fig:eval}b and d across different model sizes, confirming that the behaviors are not specific to Pythia-12B.

\begin{figure}
\begin{center}
\includegraphics{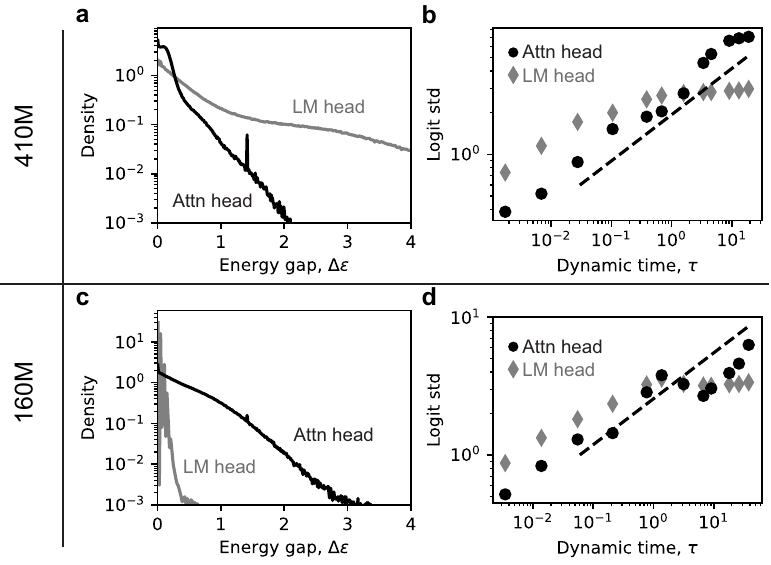}
\end{center}
\caption{Energy gap density (left column) and attention logit coldness versus $\tau$ (right column) for Pythia-410M (a--b) and Pythia-160M (c--d), replicating Figure~\ref{fig:eval}b and d. In each energy gap panel, black shows a representative attention head and gray shows the LM head; both have non-zero density at $\Delta\epsilon \to 0$. In each coldness panel, circles show an attention head and diamonds show the LM head; the dashed line has slope $+1/3$ for reference. Analysis follows Appendix~\ref{sec:C.1.2}.}
\label{fig:pythia-gap-coldness}
\end{figure}

Figure~\ref{fig:rawloss} illustrates the loss fitting procedure described in Appendix~\ref{sec:C.1.2} for all eight Pythia model sizes.
The raw cross-entropy loss $L$ (left panel) does not decay to zero as training progresses: each model converges toward a strictly positive floor, the irreducible loss $L_{\backslash\tau}$, which decreases with model size.
Fitting $L$ as a plain power law without this constant would therefore confound the training-dependent decay with model-size offsets and yield unreliable exponents.
Instead, we fit $L(\tau) = c_\tau\,\tau^{-\alpha_\tau} + L_{\backslash\tau}$ with $L_{\backslash\tau}$ as a free parameter, which correctly separates the two contributions.
After subtracting the fitted $L_{\backslash\tau}$ (right panel), the training-dependent component $L - L_{\backslash\tau}$ collapses across all eight model sizes onto a single power law close to $1/3$, validating the fitting procedure and confirming that the exponent is approximately size-independent.

\begin{figure}
\begin{center}
\includegraphics{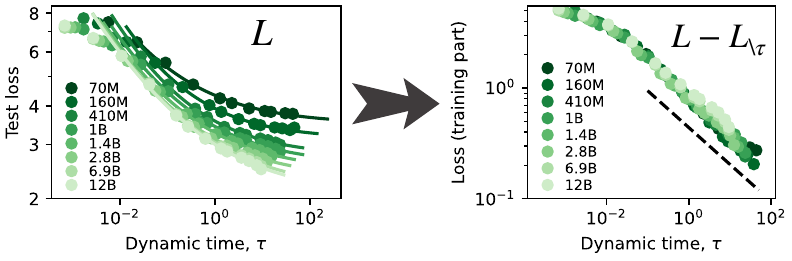}
\end{center}
\caption{The raw cross-entropy loss and the fitting procedure for all eight Pythia model sizes (70M to 12B), extending Figure~\ref{fig:eval}c. Left: raw loss $L$ versus $\tau$ with power-law-plus-constant fits overlaid (solid lines). The loss converges to a strictly positive floor $L_{\backslash\tau}$ that decreases with model size, requiring the constant to be included in the fit. Right: training-dependent component $L - L_{\backslash\tau}$ after subtracting the fitted irreducible loss. Curves from all model sizes collapse and follow a single $\sim\tau^{-1/3}$ power law (dashed reference line). Analysis follows Appendix~\ref{sec:C.1.2}.}
\label{fig:rawloss}
\end{figure}

Figure~\ref{fig:olmo2} replicates the full analysis of Figure~\ref{fig:eval} for OLMo-2, a different model family with a distinct architecture and training setup (\path{olmo/eval_attn_logit_std.py}, \path{olmo-step-tau.py}, and OLMo-2 variants of the analysis notebooks).
Panel a shows the per-head entropy of OLMo-2-13B at the final checkpoint.
Unlike Pythia, where many later-layer heads fall well below the entropy threshold, later-layer heads in OLMo-2-13B fall at most slightly below it.
Like Pythia, however, the heads deepest in the low-temperature regime appear in early layers.
We highlight one such head at layer 1 with a blue square: although the mean entropy trend still decreases with layer depth, it is the early layers that most clearly satisfy the low-temperature condition.
This difference may reflect the distinct architectural choices of OLMo-2, including grouped-query attention, RMS normalization applied to query and key projections before RoPE, and larger peak learning rates compared to Pythia.
Despite this architectural difference in which layers host the low-temperature heads, the theoretical conditions remain satisfied.
Panel b shows that the energy gap density of the highlighted layer-1 head is non-zero at $\Delta\epsilon \to 0$, and panel d shows that its logit coldness grows as $\tau^{1/3}$ at late training, consistent with the theory's prediction for a bottleneck softmax.
Based on the theory, the existence of such a head, regardless of its layer, is sufficient to drive $1/3$ scaling in the total loss.
Panel c confirms this prediction: the fitted loss exponents for OLMo-2-1B, 7B, and 13B are approximately $0.35$, $0.30$, and $0.36$, all close to $1/3$.
These results establish that the $1/3$ scaling law generalizes across model families, even though details like which head has the lowest temperature may differ.

\begin{figure}
\begin{center}
\includegraphics[width=\linewidth]{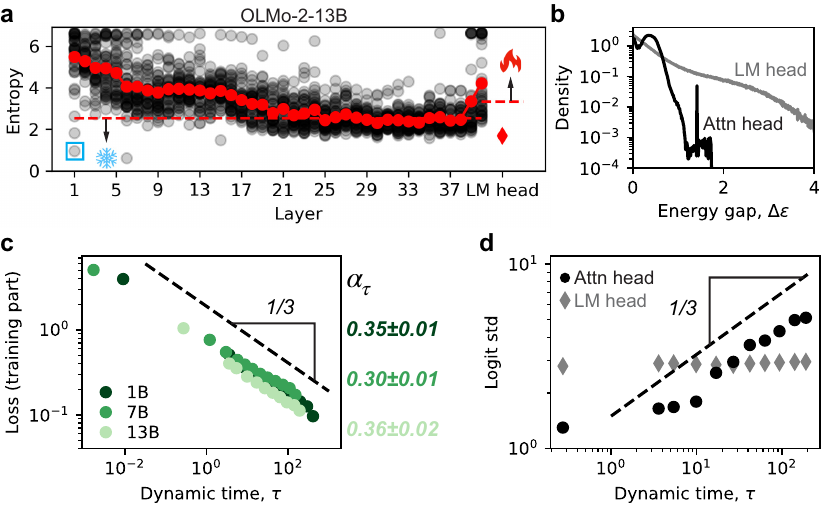}
\end{center}
\vskip -0.1in
\caption{Replication of the Figure~\ref{fig:eval} analysis for OLMo-2, extending the results to a different model family. (a) Per-head entropy of OLMo-2-13B at the final checkpoint. The red diamond shows the upper bound of the LM head target entropy. Although the mean entropy decreases in later layers, later-layer heads fall at most slightly below the entropy threshold (horizontal dashed line), and the heads deepest in the low-temperature regime appear in early layers. The layer-1 head highlighted by a blue square is in the low-temperature regime. (b) Energy gap density for the highlighted layer-1 head; the density is non-zero at $\Delta\epsilon \to 0$. (c) Loss versus dynamic time $\tau$ for OLMo-2-1B, 7B, and 13B, with fitted exponents $0.35$, $0.30$, and $0.36$. (d) Attention logit coldness of the highlighted layer-1 head grows as $\tau^{1/3}$ (dashed reference line), while the LM head saturates. Analysis follows Appendix~\ref{sec:C.1.2}.}
\label{fig:olmo2}
\end{figure}

\subsection{Training with MSE}\label{app:train}

\subsubsection{Training methods}

To test whether the LM head is the sole bottleneck for the $1/3$ loss scaling, we train a Pythia-160M student using MSE between final logits, bypassing the output softmax entirely.
The student is randomly initialized using the Pythia-160M architecture.
The teacher is a pretrained Pythia-160M model at its final checkpoint (step $79{,}000$), held frozen throughout.
The training objective is MSE between mean-centered logits of the student and the teacher, averaged over tokens:
\begin{equation}
    L_{\rm MSE} = \left\langle \frac{1}{n}\|\tilde{y}_s - \tilde{y}_t\|^2 \right\rangle,
\end{equation}
where $\tilde{y}_s = y_s - \bar{y}_s$ and $\tilde{y}_t = y_t - \bar{y}_t$ are the mean-centered logit vectors of the student and teacher ($\bar{y}$ denotes the vocabulary mean), $n$ is the vocabulary size here, $\langle\cdot\rangle$ denotes averaging over tokens, and the teacher temperature is $1$ (\texttt{loss.py}).
Mean-centering makes the loss invariant to global logit shifts, which do not affect any softmax output and are therefore irrelevant to the distribution being learned.

The student is trained for $79{,}000$ steps with peak learning rate $\eta_{\rm peak} = 6\times 10^{-4}$, linear warmup over $1430$ steps, and cosine decay over $143{,}000$ steps, the same schedule shape as Pythia (\texttt{train.py}).
Training uses the Pile (deduped) dataset with a batch size of $256$ sequences of $2048$ tokens (approximately $0.5\times 10^6$ tokens per step) on $8$ H200 GPUs.
Checkpoints are saved at the same $12$ steps as the Pythia evaluation in Appendix~\ref{sec:C.1.1}: $128, 256, \ldots, 79{,}000$.

\subsubsection{Analysis methods}

\paragraph{Dynamic time.}
Dynamic time $\tau$ is computed by the same cumulative-sum procedure as Appendix~\ref{sec:C.1.2}, using the Pythia learning rate schedule with $\eta_{\rm peak} = 6\times 10^{-4}$ and $t_{\max} = 143{,}000$.

\paragraph{Logit coldness.}
The per-head attention logit coldness is computed by the same procedure as in Appendix~\ref{sec:C.1.2} (\path{attn_logit_std_mse1.ipynb}).
A representative attention head (layer 1, head 4) is selected for display in Figure~\ref{fig:distill}a as filled circles.
The LM head logit coldness is the standard deviation of the pre-softmax LM logits over the vocabulary, evaluated at each checkpoint and plotted as diamonds in Figure~\ref{fig:distill}a.
A vertical line marks the dynamic time at which the LM head coldness saturates, chosen by eye as the point after which later checkpoints no longer increase, separating the early-training phase from the late-training plateau where the MSE loss fit is performed.

\paragraph{MSE loss fitting.}
The MSE training loss after the LM head saturation point is fit using the same power-law-plus-constant form as Appendix~\ref{sec:C.1.2}:
\begin{equation}
    L_{\rm MSE}(\tau) = c_\tau\,\tau^{-\alpha_\tau} + L_{\backslash\tau},
\end{equation}
with free parameters $c_\tau$, $\alpha_\tau$, and $L_{\backslash\tau}$ estimated by nonlinear least squares on $\log L_{\rm MSE}$ (\path{mse_1.ipynb}).
The fitted $L_{\backslash\tau}$ is subtracted to yield the training-dependent component $L_{\rm MSE} - L_{\backslash\tau}$, plotted in Figure~\ref{fig:distill}b together with a reference line of slope $-1/3$.

\end{document}